\documentclass[small authors, authorwidth=0.36\linewidth]{lib/paper}

\PassOptionsToPackage{numbers,sort&compress}{natbib}
\usepackage{amsthm}

\newcommand{\dmy}[1]{{\color{gray}#1}}
\usepackage{amsmath,amsfonts,bm}

\newcommand{\figleft}{{\em (Left)}}
\newcommand{\figcenter}{{\em (Center)}}
\newcommand{\figright}{{\em (Right)}}
\newcommand{\figtop}{{\em (Top)}}
\newcommand{\figbottom}{{\em (Bottom)}}

\def\eqref#1{equation~\ref{#1}}

\def\1{\bm{1}}

\DeclareMathAlphabet{\mathsfit}{\encodingdefault}{\sfdefault}{m}{sl}
\SetMathAlphabet{\mathsfit}{bold}{\encodingdefault}{\sfdefault}{bx}{n}

\def\gE{{\mathcal{E}}}

\newcommand{\pie}{\gls{pi_mathrmeff}}
\newcommand{\pip}{\gls{pi_mathrmpot}}
\newcommand{\gEe}{\gls{ge_mathrmeff}}
\newcommand{\gEp}{\gls{ge_mathrmpot}}

\newcommand{\E}{\mathbb{E}}

\newcommand{\KL}{D_{\mathrm{KL}}}

\DeclareMathOperator*{\argmax}{arg\,max}
\DeclareMathOperator*{\argmin}{arg\,min}

\DeclareMathOperator{\Tr}{Tr}

\providecommand{\cA}{\mathcal A}

\providecommand{\cE}{\mathcal E}

\providecommand{\cN}{\mathcal N}

\providecommand{\cP}{\mathcal P}

\providecommand{\cR}{\mathcal R}
\providecommand{\cS}{\mathcal S}

\providecommand{\cV}{\mathcal V}

\providecommand{\cX}{\mathcal X}
\providecommand{\cY}{\mathcal Y}
\providecommand{\cZ}{\mathcal Z}

\usepackage{aliascnt}
\theoremstyle{plain}
\newtheorem{theorem}{Theorem}[section]
\newaliascnt{proposition}{theorem}
\newtheorem{proposition}[proposition]{Proposition}
\aliascntresetthe{proposition}
\newaliascnt{lemma}{theorem}
\newtheorem{lemma}[lemma]{Lemma}
\aliascntresetthe{lemma}
\newaliascnt{corollary}{theorem}

\aliascntresetthe{corollary}
\theoremstyle{definition}
\newaliascnt{definition}{theorem}
\newtheorem{definition}[definition]{Definition}
\aliascntresetthe{definition}
\newaliascnt{assumption}{theorem}
\newtheorem{assumption}[assumption]{Assumption}
\aliascntresetthe{assumption}
\newaliascnt{example}{theorem}
\newtheorem{example}[example]{Example}
\aliascntresetthe{example}
\theoremstyle{remark}
\newaliascnt{remark}{theorem}

\aliascntresetthe{remark}

\def\cofirst{\equal}
\usepackage[glossaries]{lib/preamble}
\ifdefined\pdfsuppressptexinfo
\fi
\definecolor{reviewgreen}{rgb}{0,0.5,0}
\DeclareRobustCommand{\reviewgreen}[1]{{\color{reviewgreen}\hypersetup{allcolors=reviewgreen}#1}}
\let\reviewgreen\relax

\definecolor{figfivestateone}{HTML}{8172B3}
\definecolor{figfivestatetwo}{HTML}{DD8452}
\definecolor{figfivestatethree}{HTML}{937860}
\definecolor{figthreeteal}{rgb}{0.0009308,0.569543,0.5648127}
\definecolor{figthreeorange}{rgb}{0.8690623,0.4014344,0.0464332}

\setlist{compact}

\crefname{assumption}{Assumption}{Assumptions}
\Crefname{assumption}{Assumption}{Assumptions}

\usepackage{etoolbox}

\setlist{compact}
\definecolor{mypurple}{HTML}{7A4988}

\makeatletter
\providecommand{\theHALG@line}{}
\renewcommand{\theHALG@line}{\thealgorithm.\arabic{ALG@line}}
\makeatother

\title{The Geometry of Empowerment}

\author{%
  Catherine Ji\cofirst \\
  Department of Physics\\
  Princeton University\\
  \texttt{cj7280@princeton.edu} \\
  \And
  Vivek Myers\cofirst \\
  Department of Computer Science \\
  UC Berkeley\\
  \texttt{vmyers@berkeley.edu} \\
  \AND
  Sergey Levine \\
  Department of Computer Science\\
  UC Berkeley \\
  \And
  Benjamin Eysenbach \\
  Department of Computer Science\\
  Princeton University
}

\begin{document}

\maketitle

\begin{abstract}
    Empowerment captures the capacity for an agent to actively control its environment.
    While conceptually appealing as an information-theoretic quantity, the connection between empowerment and structurally central states that provide broad access to future outcomes has remained an open question.
    In this work, we link empowerment maximization and skill-learning methods to provide new geometries for interpreting and analyzing empowerment.
    Our analyses answer longstanding open questions on the connections between empowerment and structural centrality.
    Our analyses also reveal distinctions between information and reward geometries, highlighting important theoretical implications to build scalable empowerment-maximization methods.\footnotemark[1]
\end{abstract}

\footnotetext[1]{Website \& Code: \url{https://empowerment-geometry.github.io/}}

\section{Introduction}

\label{sec:intro}
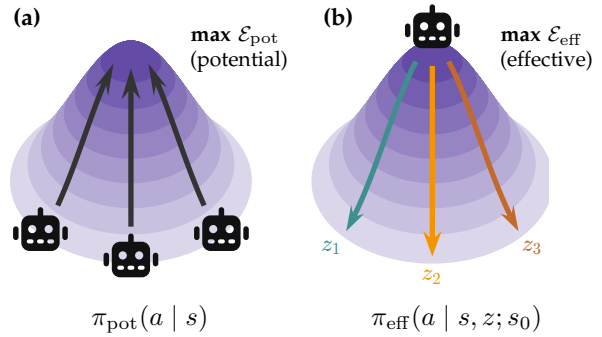
\begin{wrapfigure}{R}{0.5\linewidth}
    \centering
    \begin{tikzpicture}
        \node { \includegraphics[width=0.9\linewidth]{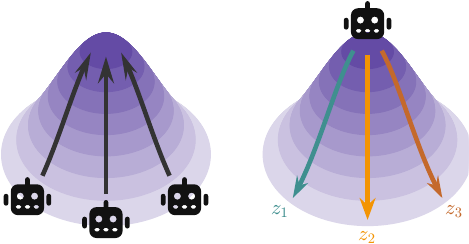} } ;
        \node[%
            anchor=north west,%
            xshift=2.4cm,%
            yshift=-0.3cm,%
            text width=4em,%
            font=\raggedright\footnotesize,
        ] at (current bounding box.north west) {\parbox{\linewidth}{ \textbf{max} $\gE_{\mathrm{pot}}$\\ (potential) }};
        \node[%
            anchor=north west,%
            xshift=6.5cm,%
            yshift=-0.3cm,%
            text width=4em,%
            font=\raggedright\footnotesize,
        ] at (current bounding box.north west) {\parbox{\linewidth}{ \textbf{max} $\gE_{\mathrm{eff}}$\\ (effective) }};
        \node[anchor=north west, xshift=0.05cm, yshift=-0.05cm, font=\small\bfseries] at (current bounding box.north west) {(a)};
        \node[anchor=north west, xshift=4.15cm, yshift=-0.05cm, font=\small\bfseries] at (current bounding box.north west) {(b)};
    \end{tikzpicture}
    \mbox{%
        \hsize .5\linewidth
        \centerline{$\pi_{\mathrm{pot}}(a \mid s)$}
        \centerline{$\pi_{\mathrm{eff}}(a \mid s, z; s_0)$}%
    }

    ~
    \begin{subfigure}[t]{.4\linewidth}
    \end{subfigure}\hfill%
    \begin{subfigure}[t]{.45\linewidth}
    \end{subfigure}
    \caption{%
        \textbf{Two trained empowerment-maximizing policies characterize central states in an MDP:} Consider an agent (robot) on a hill that is easier to descend than to ascend.
        In Fig. 1(a), the potential policy ascends the empowerment landscape to the top of the hill, using $\gEp(s)$ as an intrinsic reward.
        In Fig. 1(b), The effective policy $\pie$ maximizes $I^\pi(Z; S_+ \mid S_0)$, associating each skill $z$ with a downhill route.%
    }
    \label{fig:two_policies}
\end{wrapfigure}

An intelligent agent should be in control of its future. For example, a well-trained hiker at a mountain can (1) choose and traverse down one of many trails from a mountain peak (\cref{fig:two_policies}(b)) {\it or} (2) navigate to and stay at the peak (\cref{fig:two_policies}(a)). While both useful, these behaviors are distinct: the first behavior characterizes ``actively exerting control'' while the latter characterizes ``being in control.''

Cleanly parameterizing, optimizing, and separating these behaviors requires a tractable mathematical formulation of controllability.
\textbf{Empowerment} is an information-theoretic object that quantifies controllability in terms of bits.
As originally defined in~\citep{klyubin2005empowerment}, \textbf{potential empowerment} measures the maximal bits an agent's actions can inform its future state from some current state.
In other words, the potential empowerment measures how much control an agent \textit{could} command from a given state.
An alternative measure of empowerment, termed here as \textbf{effective empowerment}, is a mutual information that quantifies how readily an agent's policy \textit{actively} informs its future state~\citep{capdepuy2011informational}.
While sometimes conflated, these two objects are different. In the example of the mountain in \cref{fig:two_policies}, a potential empowerment policy ascends the mountain to a state from which the agent has options. Meanwhile, the effective empowerment policies actively exercise options and descend the mountain in different directions.

\paragraph{Empowered States have Broad Access.} Prior works conjecture that empowerment is a generic way for agents to do goal-oriented tasks~\citep{salge2014empowerment,schneider2025informationtheoreticpolicypretrainingempowerment}. \citet[\S4.3.3]{salge2014empowerment} call this the AI Empowerment Hypothesis:
\begin{quote}
    \itshape Empowerment provides a task-independent motivation that generate AI behaviour which is beneficial for a range of goal-oriented behaviour.
\end{quote}
However, empowerment does not have immediate obvious connections to downstream task performance.
Indeed, while early works in empowerment suggest a link between empowerment and preferred or central states~\citep{anthony2008preferred,jung2011empowerment,salge2014empowerment}, to the best of our knowledge, formally characterizing notions of MDP centrality and adaptation in relation to empowerment has remained a longstanding open problem.
In this work, we formally show high potential empowerment states lie at centers of reachable regions in the state occupancy simplex (\cref{thm:effective_info_geom}).
The empowerment objective and long horizon generalizations also admit exact hitting time interpretation (\cref{lem:bottleneck}) and capture centrality in a temporal distance geometry~\citep{myers2025learning}.
In addition to providing intuition, these notions of centrality directly lower bound an agent's ability to adapt to downstream tasks in tabular settings (\cref{thm:width_linear_sharp}), providing a partial answer to Salge's longstanding hypothesis.
Thus, potentially empowered agents can reach central states from which they can adapt, where effective empowerment quantifies the adaptation.

\paragraph{Information Geometries and Reward Geometries can Diverge in Continuous-State \gls{mdp}s.} The above results theoretically answer questions in the empowerment literature: (1) empowered states are central states and (2) in discrete settings, empowered states are states from which agents can adapt to downstream tasks. Attempting to extend the adaptation result to continuous-state \gls{mdp}s  reveals that information and reward geometries can diverge: an arbitrarily empowered state can provide arbitrarily small adaptation in a continuous MDP with continuous rewards (\cref{prop:mi_lipschitz_collapse}). The same misalignment appears in discrete settings once the reward prior is anisotropic (\cref{prop:anisotropic_degraded}). Thus, empowered agents are \textit{not generally} adaptive agents in the continuous-state \gls{mdp}s. Scalable formulations of empowerment maximization must address this discrepancy between information and reward adaptation geometries.

Taken together, our results give an information-geometric centrality theorem for the empowerment objective and a complementary temporal interpretation over skill rollouts. These theoretical interpretations provide an answer to longstanding questions about the relationship between empowerment, centrality and access, and adaptation to downstream reward functions: in tabular settings, empowered agents navigate to states centrally-located in an information geometry from which they can adapt. When considering continuous-state \gls{mdp}s and/or anisotropic reward priors, this intuition no longer holds. Thus, our theoretical results show that scalably unifying empowerment and adaptation requires alternative formulations that align information and reward geometries.

\section{Related Work}
\label{sec:related}
As originally introduced, the empowerment at a state is the maximum possible information that can be communicated from open-loop action sequences to future states~\citep{klyubin2008keep}, though closed-loop variants have also been discussed~\citep{salge2014empowerment,mohamed2015variational}.
An empowerment maximizing agent should maximize the empowerment of its occupied states.
~\citet{capdepuy2011informational} distinguishes this channel capacity form of empowerment, termed {\it potential empowerment}, from the mutual information between the action sequences and a future state \emph{realized} by a specific policy, termed {\it actualized empowerment}. In this work, we refer to this form of empowerment as {\it effective empowerment}. These quantities are formalized in \ref{sec:prelims}.

Where should an agent go if told that it will be asked to solve some soon-to-be-specified task? One approach is to consider the geometry of the state space and navigate to \textbf{central, connected states} or \textbf{bottleneck states} within the MDP: states from which the agent can easily visit many other states~\citep{sabidussi1966centrality}. Such notions of MDP centrality have been studied in both reinforcement learning~\citep{jung2011empowerment, anthony2008preferred, csimcsek2007betweenness, turner2023optimalpoliciestendseek} and neuroscience~\citep{stachenfeld2014design} community, often in the context of exploration.
A related line of work studies whether optimal behavior under task uncertainty implies being at central states and actions. In particular, \citet{turner2023optimalpoliciestendseek} show that, under symmetry assumptions on the \gls{mdp} and uncertainty in the reward function, optimal agents prefer state-actions with many future options.
While these notions and definitions of centrality and bottlenecks are appealing, these objects are only well-defined in settings with discrete states and actions.

Prior empowerment literature has conjectured that potential empowerment measures the degree of centrality of a state in an MDP~\citep{anthony2008preferred, jung2011empowerment}. For example, \citet{salge2014empowerment} note that potential empowerment exactly corresponds to the log of the total reachable states in discrete, deterministic settings. In stochastic settings, prior works discuss the relationship between empowerment and closeness centrality and show correlations between these quantities in simple didactic experiments; however, no formal link between centrality and empowerment is provided.
\citet{jung2011empowerment} mention the possibility that empowerment might identify ``gateway'' states, but do not explicitly test nor quantify this observation in experiments. Our work provides a rigorous grounding to these observations, formally connecting potential and effective empowerment to two different notions of centrality: centrality in temporal distance geometry and centrality in the information geometry.

One important consequence of our theoretical results is a notion of centrality that remains well defined in continuous state spaces.
A natural generalization of discrete measures of reachability and closeness is the temporal distance~\citep{myers2025learning}. In discrete settings, the temporal distance corresponds to a simple function of the hitting time from state-to-state. Previous work has shown temporal distances can effectively guide exploration~\citep{jiang2025etd} and goal-reaching~\citep{myers2025offlinegoalconditionedreinforcementlearning, park2026dualgoalrepresentations} in continuous, high-dimensional tasks. To our knowledge, there has not been prior work utilizing temporal distances to identify and characterize centralized or bottleneck states. Here, we show that empowerment corresponds to the reduction in time/temporal distance to future states.

\paragraph{Estimating and Maximizing Empowerment in Practice.}
There are two different forms of empowerment: potential empowerment and effective empowerment~\citep{capdepuy2011informational}.
While prior works discuss maximizing both quantities and their variants~\citep{levy2025representation,gregor2017variational,ringstrom2022reward,schneider2025informationtheoreticpolicypretrainingempowerment},
to the best of our knowledge, there have been no prior works that scalably maximize potential empowerment in high-dimensional, continuous-state settings.
Meanwhile, there has been substantially more prior work in RL estimating and maximizing \emph{effective empowerment} (or similar) quantities in high-dimensional, continuous settings, in a line of work termed \gls{misl}.
Prior works use mutual information objectives to learn a latent $z$-conditioned policy $\pi(a \mid s, z)$ that transmits bits about an input vector $Z$ to future states $S_+$ via its actions (i.e., closed-loop control~\citep{gregor2017variational, eysenbach2018diversity, zheng2025misl, park2023metra}).
These methods learn {\it skills} $\pi(a \mid s,z)$ that actively transmit information to the future, which are high effective empowerment policies.

By making use of contrastive and variational estimators of mutual information, prior works scale these methods to learn reactive policies in high-dimensional (e.g., with images) continuous-control tasks~\citep{zheng2025misl}.
While some of these skill-learning methods mention empowerment~\citep{gregor2017variational}, they do not directly optimize potential empowerment as defined in~\citep{klyubin2005empowerment} and instead only optimize effective empowerment. Recently, empowerment and its variants have also served as objectives for assistance~\reviewgreen{\citep{myers2024learningtoassist,ellis2026training,yang2026when}}. While we do not present experiments in high-dimensional or continuous settings, we highlight these lines of work to emphasize that empowerment and related objectives underpin state-of-the-art intrinsic rewards for unsupervised \gls{rl} and new assistance methods. Understanding the theoretical reasons to maximize or modify empowerment therefore directly informs the design of these important methods.

\section{Preliminaries}
\label{sec:prelims}
We will define notions of empowerment in terms of occupancy measures over \gls{mdp}s, following a long line of work that casts decision making as probabilistic inference ~\reviewgreen{\citep{barreto2017successor,eysenbach2020c,dayan1993improvinga,gregor2017variational,fu2018variational}}.
\subsection{Discounted State Occupancy}
\paragraph{MDP problem setting.} Our analysis focuses on applying information theory to a fully-observed \gls{mdp} with states $s \in \cS$, actions $a \in \cA$, and environment dynamics $p(s' \mid s, a)$. Let the skill and starting state-conditioned policy be the function $\pi(a \mid s, z; s_0)$ with skill $z \in \cZ$.
The latent skill $z$ lives in a continuous skill space $\cZ \subseteq \mathbb{R}^d$ (i.e. the surface of a hypersphere), including in tabular \glspl{mdp}. The policy $\pi(a \mid s, z; s_0)$ is a distribution over actions given the current state $s$, the latent $z$, and the start state $s_0$. We write $\Pi_\cZ$ for this class of skill- and starting-state conditioned policies, and $\Pi$ for the class of ordinary policies $\pi(a \mid s)$ that carry no skill. Fixing a skill in a policy of $\Pi_\cZ$ leaves an ordinary policy in $\Pi$. The start-state $s_0$ is the context of a skill~\citep{klyubin2008keep}.

\paragraph{Reachable state-occupancy polytope.}
For some initial state $s_0$, skill latent $z \in \cZ$, and skill and start-state-conditioned policy $\pi(a \mid s, z; s_0)$, the \gls{dsom} takes the form
\begin{equation}
    p^\pi(s_+ = s \mid s_0, z) \triangleq (1-\gamma)\sum_{t=0}^{\infty} \gamma^t p^\pi(s_t = s \mid s_0, z),
    \label{eq:dsom}
\end{equation}
where $p^{\pi}(s_t = s \mid s_0, z)$ is the probability, under policy $\pi$ and the \gls{mdp} dynamics, visits state $s$ after $t$ time steps when initialized at state $s_0$ under skill $z$. Equivalently, the \gls{dsom} is the probability measure of the occupied future state $S_+$ at a geometric termination time $\tau \sim \text{Geom}(1-\gamma)$.

Averaging over the skill distribution $p(z \mid s_0)$ gives the skill-marginal \gls{dsom},
\begin{equation}
    p^\pi(s_+ \mid s_0) = \E_{p(z \mid s_0)}\big[p^\pi(s_+ \mid s_0, z)\big],
    \label{eq:dsom_marginal}
\end{equation}
the future-state distribution induced by the policy when the skill is unknown.

Every \gls{dsom} is a distribution over states within the simplex $\Delta(\cS)$.
For a fixed \gls{mdp}, an agent can only realize a strict subset of \gls{dsom} in the simplex.
For a fixed start state $s_0$, the realizable \gls{dsom}s  form the \emph{reachable state-occupancy polytope}
\begin{equation}
    \cP(s_0) \triangleq \bigl\{\, p^\pi(s_+ \mid s_0) \;:\; \pi \in \Pi \,\bigr\} \;\subseteq\; \Delta(\cS),
    \label{eq:polytope}
\end{equation}
a convex subset of the simplex whose linear constraints are the Bellman flow equations (see the polytope inside the simplex in \cref{fig:example-left})~\citep{eysenbach2021information}.
We write $\cP(s)$ for the polytope reachable from a generic start state $s$.
Our analysis in \cref{sec:geometry} studies the geometry of exactly this polytope.

\subsection{Empowerment Preliminaries}

We consider two forms of empowerment: potential (the maximal empowerment that \textit{could} be achieved at a state) and effective (the actualized influence of a particular skill-conditioned policy (i.e., collection of policies) at a given state on the future state distribution).

\paragraph{Potential and Effective Empowerment.}
Our work highlights the difference between empowerment, as originally defined in~\citet{klyubin2005all} and the \gls{mi} objects directly maximized in \gls{misl} methods~\citep{mohamed2015variational, eysenbach2018diversity, park2023metra, zheng2025misl}.
We term these forms of empowerment as potential and effective empowerment, respectively.
The distinction between these forms of empowerment is not new and has been pointed out in prior work~\citep{capdepuy2011informational}, though (to our knowledge) modern works do not always differentiate between these objects~\citep{mohamed2015variational}.

\citet{klyubin2005all} originally defines (potential) empowerment as the degree to which action sequences of length $H$ from a current state (i.e. input of the information channel) control the future state (i.e. output). Formally,
\begin{align*}
    \gEp_H(s)=\max_{p(a_0^{H-1})} I(A_0^{H-1};S_H\mid S_0=s)
\end{align*}
where $A_0^{H-1} \triangleq (A_0, \dots, A_{H-1})$ denotes the action sequence. In our work, we make two changes. First, rather than considering entire sequences of actions as inputs (i.e. open-loop controls), we instead consider latents $z$ that parameterize a \emph{reactive} policy. We refer to these skill-conditioned, closed-loop controls as ``skills''. Potential empowerment then measures the maximal influence these skills exert on a future state. Second, rather than measure these skills' influence on a state $H$ steps in the future, we measure how much skills control states at a geometric stopping time, so the discount factor sets the timescale of the empowerment.
\begin{definition}[Potential empowerment] Let policies $\pi \in \Pi_\cZ$ condition on skill latents $z \in \cZ$ and let $S_+$ denote the random state $s_+ \sim p^\pi(s_+ \mid s_0, z)$ sampled from the \gls{dsom}.
    Potential empowerment is the channel capacity over a {\it family of policies} dependent on an initial state~\citep{klyubin2005empowerment,capdepuy2011informational}
    \begin{align*}
        \gEp(s_0) & = \max_{\pi \in \Pi_\cZ} I^\pi (Z;S_+ \mid S_0=s_0).
    \end{align*}
    \label{def:pot}
\end{definition}
\begin{definition}[Potential empowerment policy]
    A potential empowerment policy $\pip \in \Pi$ that takes the form $\pip(a \mid s)$ maximizes the potential empowerment as a reward:
    \begin{align*}
        \pip^* = \argmax_{\pi \in \Pi} \E_{\pi}\left[\sum_t \gamma^t \gEp(s_t)\right].
    \end{align*}
    \label{def:potential}
\end{definition}
While conceptually similar, mutual information skill learning methods do not maximize potential empowerment, and maximize effective empowerment. {\it Effective empowerment} captures the \gls{mi} itself:
\begin{definition}[Effective empowerment]
    Consider skill and start state-conditioned policies $\pi \in \Pi_\cZ$ that take the form $\pi(a \mid s, z; s_0)$. {\it Effective empowerment} is the mutual information realized by a particular skill and start state-conditioned policy at a particular starting state $s_0$ under the uniform skill source $p(z \mid s_0) = \mathrm{Unif}(\cZ)$:
    \begin{align*}
        \gEe(\pi, s_0) & = I^\pi (Z;S_+ \mid S_0 = s_0) = \E_{p(\dmy{z} \mid s_0)} \KL[p^\pi(\dmy{s_+} \mid s_0, z) \| p^\pi(\dmy{s_+} \mid s_0)]
    \end{align*}
    Correspondingly, we refer to the skill and start state-conditioned policy $\pi(a \mid s,z ; s_0)$ as the \emph{effective empowerment policy}.

    This objective is similar to the objectives directly maximized in modern skill-learning algorithms~\reviewgreen{\citep{gregor2017variational, eysenbach2018diversity, zheng2025misl}}; note that \gls{misl} methods generally assume a fixed starting state $s_0$, so the learned skills are specific to $s_0$. Thus, we also refer to the effective empowerment policies as skills. \label{def:effective}
\end{definition}
Importantly, effective empowerment involves no maximization and measures the \gls{mi} \emph{achieved} by a particular skill-conditioned policy $\pi \in \Pi_\cZ$ that begins from a particular starting state $s_0$. This \gls{mi} depends on both the \emph{channel} $z \mapsto p^\pi(\dmy{s_+} \mid s_0, z)$, determined by the policy, and the \emph{source} $p(z \mid s_0)$ set to uniform. Potential empowerment is the maximum-achievable effective empowerment for a given starting state $s_0$, so that
\begin{align}
    \gEp(s_0) = \max_{\pi \in \Pi_\cZ}\; \max_{p(z \mid s_0)} I^\pi(Z;S_+ \mid S_0 = s_0).
    \label{eq:pot_is_max_eff}
\end{align}
Under a uniform skill prior, the policy can represent any skill source by simply relabeling skills (\cref{lem:uniform_continuous}). Fixing the source $p(z \mid s_0)$ to uniform therefore loses no generality. We include a full statement of this result in \cref{app:proof_skill_selection}.

Thus, a learned skill set (effective empowerment policies) acts as a probe of the potential empowerment at a starting state. If skills initialized at $s_0$ induce easily distinguishable future occupancies $p^\pi(s_+ \mid s_0, z)$, then $s_0$ affords many controllable futures and has high potential empowerment $\gEp(s_0)$.

\paragraph{Notation convention.} Throughout the paper, we write $\pi$ for the effective policy when context is clear and $\pie$ only when both policies appear in the same expression. The potential policy is always written as $\pip$. The object $\cE (s)$ always refers to potential empowerment measured at starting-state $s$.

\section{The Geometry of Empowerment}
\label{sec:geometry}
Prior work has empirically observed a close connection between geometry and empowerment~\citep{anthony2008preferred, jung2011empowerment}, wherein states that are more centralized in an \gls{mdp} have higher empowerment.
However, to the best of our knowledge, such connections have not been proven theoretically.
In this section, we show that empowerment is closely related to a probabilistic notion of centrality in information geometry~\citep{polyanskiy2022information} and temporal distance geometry~\citep{myers2025learning}:
\begin{itemize}
    \item \cref{subsec:geom}: Effective empowerment corresponds to the information-theoretic distance between states a \textit{ specific} skill-conditioned policy may reach in the future.
          Potential empowerment is the maximum over these distances from a starting state and thus measures the radius of a reachable space.
    \item \cref{subsec:bottlenecks}: High empowerment states are central in temporal distance geometry.
          From a current state, the potential empowerment measures the average time saved to reach future states when committing to a skill.
\end{itemize}

We then connect empowerment to adaptation on downstream tasks (\cref{subsec:downstream}). We show a positive result, a lower bound on adaptation, but also prove important limitations that highlight the fundamental difference between information and reward geometries:
\begin{itemize}
    \item \cref{subsec:downstream}, \cref{thm:width_linear_sharp}: The potential empowerment of a state lower bounds adaptation to unknown reward functions drawn from an isotropic distribution.
    \item \cref{subsec:downstream}, \cref{prop:mi_lipschitz_collapse}: In continuous settings, a state can have unbounded potential empowerment while providing arbitrarily small adaptation to a continuous reward.
    \item \cref{subsec:downstream}, \cref{prop:anisotropic_collapse}: For anisotropic reward distributions, an empowered state can provide no adaptation.
\end{itemize}

\subsection{Information Geometry of Potential and Effective Empowerment Policies.}
\label{subsec:geom}

\begin{figure}
    \centering
    \includegraphics[width=0.7\linewidth]{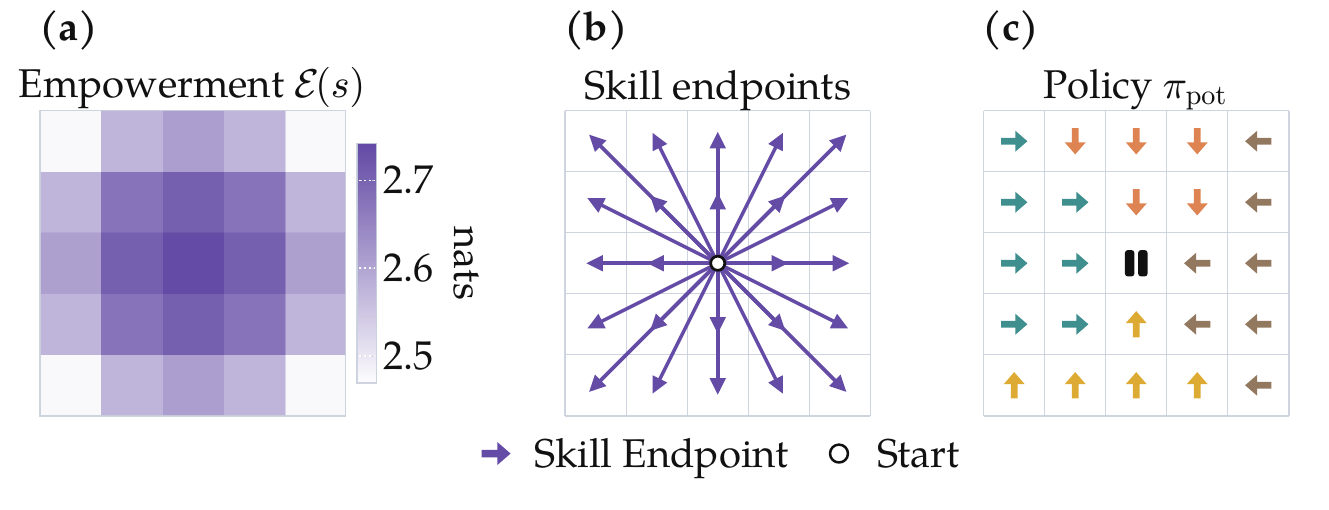}
    \caption{\textbf{Skills measure the centrality of a state.} We visualize the potential empowerment field $\cE(s)$ (a) for each state in a tabular GridWorld. The central state has high potential empowerment, because (b) the central state gives an agent access to a maximal set of distinguishable skills. A potential empowerment policy ascends the potential empowerment landscape to the center state in (c), from which it can execute different skills. See \cref{thm:effective_info_geom} for the formal statement. See \cref{app:tab_exps} for tabular experiment details.}
    \label{fig:eff_emp_5x5}
\end{figure}

\paragraph{Geometric intuition.}
Prior work has found empirically that high empowerment states are central states in an MDP, relating empowerment to the volume of reachable states in discrete, deterministic settings~\citep{salge2014empowerment}. Here, we extend definitions of centrality to the full MDP setting with stochasticity and continuous states.

Effective empowerment is a mutual information. In the dual form, effective empowerment maximization finds state occupancy distributions that lie at the center of extremized state distributions. Then, potential empowerment-seeking policies search for the states that lie at the centers of the largest radii reachable polytopes.

Before presenting formal results, we visualize an example of the information geometry of empowerment in \cref{fig:example}. Let the triangle visualize the state occupancy simplex reachable from different states $s_1$ and $s_2$. Maximizing the mutual information (effective empowerment) for a given starting state corresponds to learning skills that reach vertices of the set of reachable state occupancies. The agent maximizes effective empowerment separately at every possible starting state (orange for $s_1$, teal for $s_2$) and learns skills (orange and teal arrows) that span reachable regions. Then, to maximize potential empowerment, an agent searches among these starting states for the largest reachable region. In this case, the teal state $s_2$ corresponds to the state with the highest potential empowerment and occupies a more centralized point within the information geometry.

\paragraph{High Empowerment States are Central in Information Geometry.}
Formally, consider the state probability simplex $\Delta(\mathcal{S})$, where \gls{dsom}
are elements of the simplex $p^\pi(s_+ \mid s_0, z) \in \Delta(\mathcal{S})$.
A straightforward extension of prior work~\citep{eysenbach2021information}
shows that effective empowerment maximization is a center localization problem in the
initial state $s_0$-conditioned state probability simplex, where distance is given by the
Kullback--Leibler divergence~\citep{megiddo1983weighted}.
Then, potential empowerment measures the min-max distance to the reachable state-occupancy polytope:

\makerestatable
\begin{theorem}[Potential empowerment is the \gls{kl} radius of the reachable polytope, extending {\citet[Lem.~6.1--6.2]{eysenbach2021information}}]
    \label{thm:effective_info_geom}
    Let $\cP(s_0) \subset \Delta(\cS)$ be the reachable state-occupancy polytope of \cref{eq:polytope}. Jointly maximizing the mutual information over the skill source and the skill-conditioned policy solves the KL 1-center localization problem for $\cP(s_0)$:
    \begin{align}
        \gEp(s_0)
        =
        \max_{\pi \in \Pi_\cZ,\, p(z \mid s_0)} I^\pi(Z;S_+ \mid S_0 = s_0)
         & =
        \min_{q(s_+) \in \Delta(\mathcal{S})}
        \max_{p \in \cP(s_0)} \;
        \KL\bigl(p(\dmy{s_+}) \| q(\dmy{s_+})\bigr).
        \label{eq:one_center}
    \end{align}
    Moreover, every joint optimum $(\pi, p(z \mid s_0))$ places every skill on the boundary of the optimal KL ball measured from the optimal central occupancy $q^\star$:
    \begin{align}
        \KL\bigl(p^\pi(\dmy{s_+} \mid s_0, z) \| q^\star(\dmy{s_+})\bigr)
         & = \max_{p \in \cP(s_0)} \KL\bigl(p(\dmy{s_+}) \| q^\star(\dmy{s_+})\bigr)
        \quad\text{for every skill } z.
    \end{align}
    In other words, the optimized skills span the maximal information radius of the reachable polytope given starting state $s_0$.
\end{theorem}

We prove a general version of \cref{thm:effective_info_geom} in \cref{app:proof_info_geom}. This result shows that an exploration agent maximizing potential empowerment will seek central states in the information geometry. Prior work on empowerment was correct in conjecturing that empowerment maximization has a formal connection to centrality. However, the right geometry for studying this centrality is not (say) the Euclidean metric on states, but rather the \gls{kl} divergence between state distributions induced by different skills.
\begin{figure}
    \begin{subfigure}[b]{0.45\linewidth}
        \centering\includegraphics[width=\linewidth]{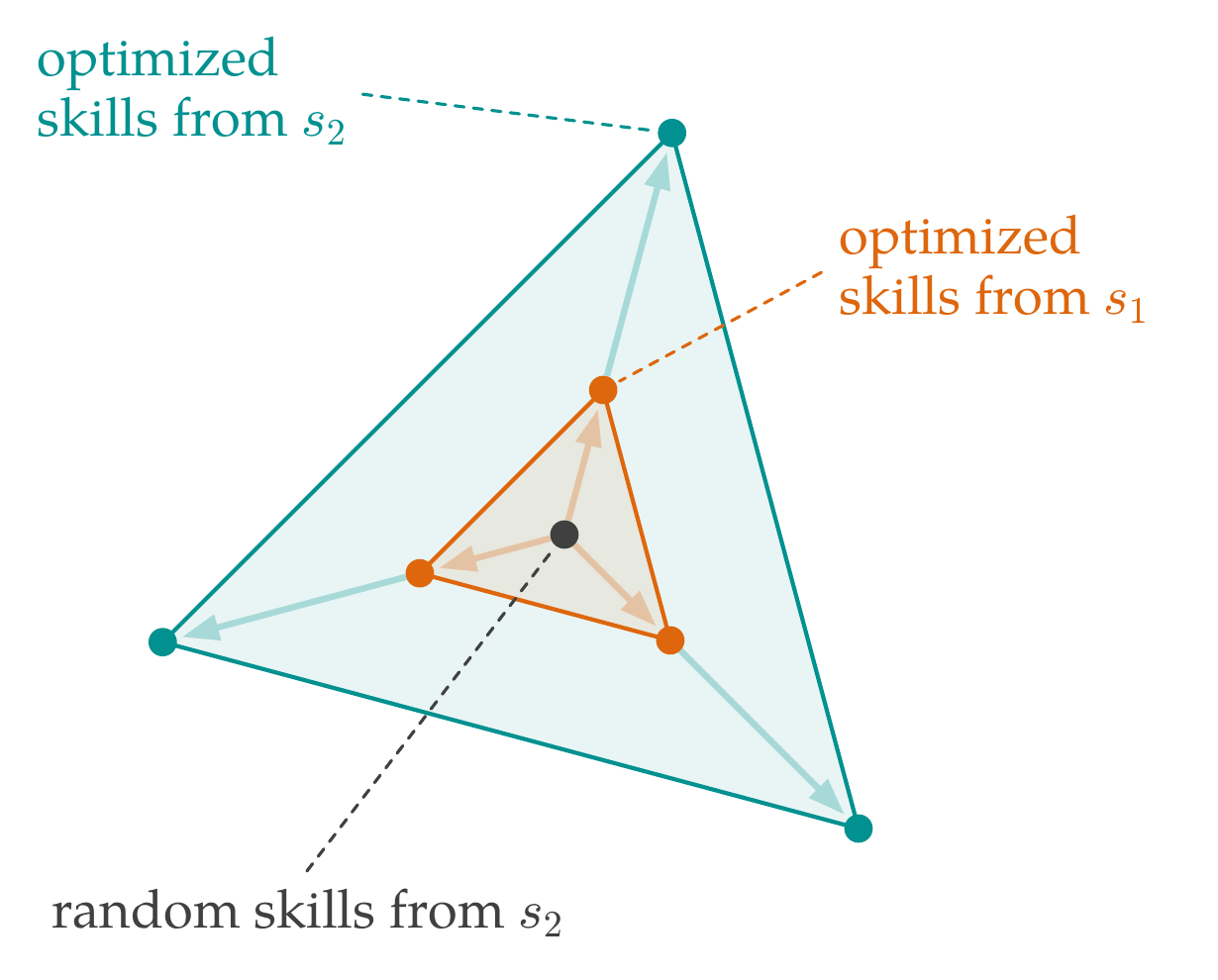}
        \caption{\textbf{State occupancy simplex}}
        \label{fig:example-left}
    \end{subfigure}
    \hfill
    \begin{subfigure}[b]{0.52\linewidth}
        \centering
        \includegraphics[width=\linewidth]{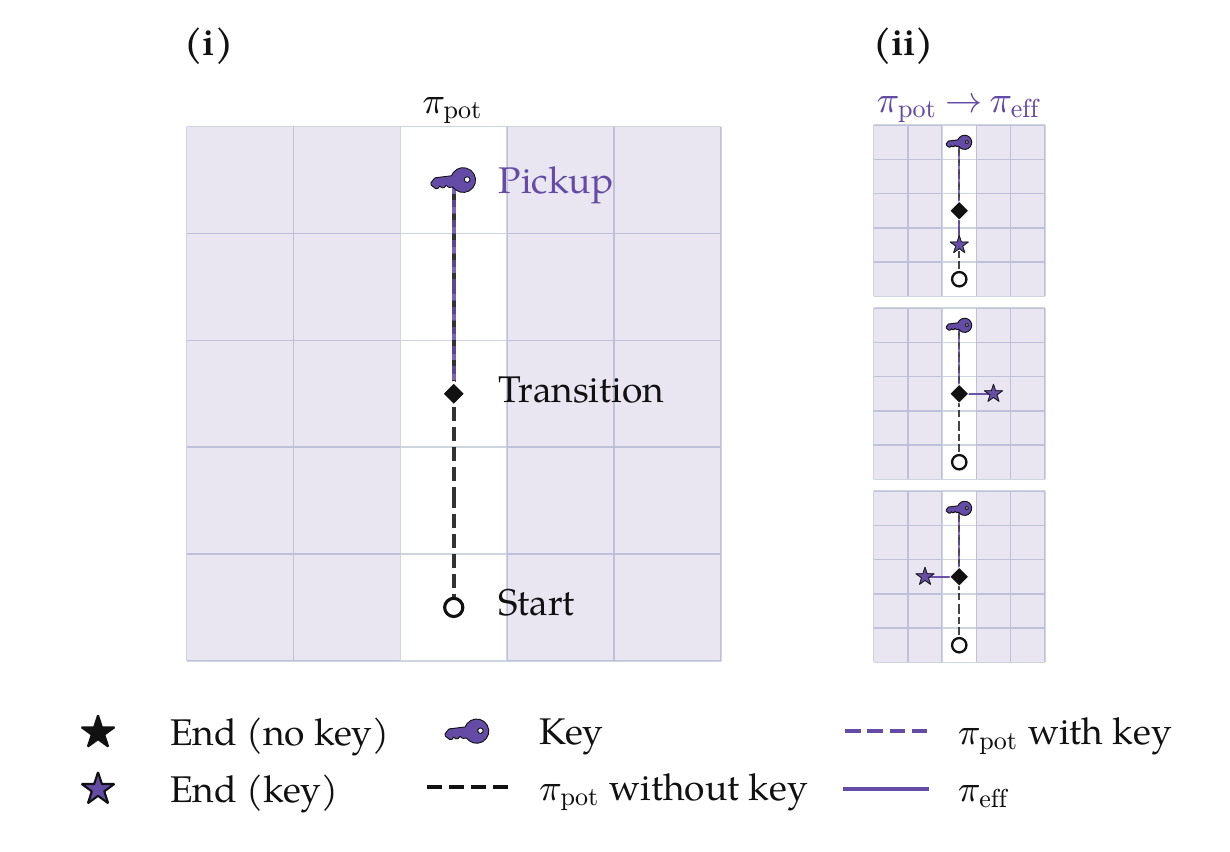}
        \caption{\textbf{Potential policy finds central states}. }
        \label{fig:example-center}
    \end{subfigure}

    \caption{{\bf Potentially empowered states are central states in the polytope.} Tabular intuition for \cref{thm:effective_info_geom}.
        (a) We can use a probability simplex $\Delta(\cS)$ to visualize the discounted state occupancy measures of each skill (circles).
        \textcolor{figthreeorange}{Orange circles} indicate the best set of skills starting from state \textcolor{figthreeorange}{$s_1$}, and \textcolor{figthreeteal}{teal circles} indicate the best skills starting from state \textcolor{figthreeteal}{$s_2$}.
        from starting states \textcolor{figthreeorange}{$s_1$} and \textcolor{figthreeteal}{$s_2$} in arbitrary MDP.
        The skills starting from \textcolor{figthreeorange}{$s_1$} define a smaller polytope (\textcolor{figthreeorange}{orange points}) in the simplex than \textcolor{figthreeteal}{$s_2$} (\textcolor{figthreeteal}{teal points}) as a result of their transition dynamics, so the potential empowerment is higher at \textcolor{figthreeteal}{$s_2$} than \textcolor{figthreeorange}{$s_1$}.
        (b) We use a GridWorld to illustrate the connection between centrality and tool use.
        An agent in the GridWorld begins at the bottom of the hallway.
        The agent can set down and pick up a key at the top of the hallway.
        The key provides the agent access to shaded regions of the \gls{mdp}.
        (b, i) A potentially empowered agent picks up the key then navigates to the center of the room, correctly identifying the middle state \emph{with the key} as the most central state.
        (b, ii) From the center, the agent can execute the maximal number of distinguishable skills.
        See \cref{app:key} for experimental details.
    }
    \label{fig:example}
\end{figure}

A simple tabular experiment demonstrates this notion of centrality (\cref{fig:example-center}). In this \gls{mdp}, the agent can pick up and drop off a key at the top of the hallway in addition to $\{\text{Left, Right, Up, Down, Stay}\}$ actions. Picking up this key allows access to the blue-shaded regions of the \gls{mdp}; without the key, the agent cannot enter these spaces in the grid. The potential empowerment policy $\pip$ that ascends the potential empowerment landscape correctly reaches the true central state: the center of the \gls{mdp} with the key. This experiment shows that states with maximal potential empowerment capture an information-geometric notion of centrality, rather than navigating towards the {\it physically} central state.

\subsection{Empowerment, Bottlenecks, and Temporal Distance}

\label{subsec:bottlenecks}

Beyond information geometry, the empowerment objective has a direct hitting time interpretation in discrete state spaces. Suppose that the agent has multiple trial trajectories to reach a future state, where each trajectory rolls out a skill from the same starting state until a geometric termination time. The empowerment measures the reduction in hitting time when the agent commits to a skill:

\begin{restatable}[Bottleneck states are high potential empowerment states]{lemma}{lemBottleneck}
    \label{lem:bottleneck}
    In a discrete \gls{mdp}, let the hitting time random variable $H^{\pie}(s_0\to s_+)$ be the number of rollouts from $s_0$ until a rollout hits $s_+$, for skills $Z \sim p(\dmy{z} \mid s_0)$. Let $H^{\pie}(s_0\to s_+\mid z)$ be the corresponding hitting time when rolling out $z$ every time. Then,
    \[
        \gEp(s_0)
        =
        \max_{\pie \in \Pi_\cZ}
        \E_{p^{\pie}(z,s_+\mid s_0)}
        \left[
            \log
            \E[H^{\pie}(s_0\to s_+)]
            -
            \log
            \E[H^{\pie}(s_0\to s_+\mid z)]
            \right].
    \]
\end{restatable}
We prove \cref{lem:bottleneck} in \cref{app:bottleneck}. In tabular settings, potential empowerment measures the reduction in time to reach future states \emph{over independent rollouts} when committing to an option/skill.

However, hitting times are not well-defined in continuous geometries, where hitting a particular state is a measure-zero event. The temporal distance is an object that remains well-defined in continuous settings and reduces to hitting times in the discrete limit \citep{myers2025learning}. In the appendix, we show that a closely related empowerment objective over \emph{contiguous} (versus repeated) skill rollouts, where the agent re-samples a skill at each geometric termination in a single trajectory, admits an exact temporal distance formulation~\citep{myers2025learning}. In that construction, the empowerment object measures the reduction in temporal distance to future states after committing to a skill (\cref{prop:rollout_temporal_distance}), recovering a hitting-time interpretation in the discrete, deterministic limit.

Conceptually, these temporal formulations capture the bottleneck effect. From a bottleneck, an agent must commit to different skills in order to quickly reach different downstream regions which is conceptually the reduction in temporal distance and hitting times considered in our results. Tabular experiments in \cref{fig:bottleneck} illustrate this effect~\citep{stachenfeld2014design,csimcsek2005identifying}. Adding slim openings (``bottlenecks'') to the grid shifts the high empowerment states from the original central state to the bottleneck openings, prioritizing states where committing to a skill (e.g. move left or move right) quickly reaches different regions. This interpretation also connects empowerment to closeness centrality~\citep{anthony2008preferred}.
\begin{figure}
    \centering
    \begin{minipage}[t]{0.48\linewidth}
        \centering
        \includegraphics[width=\linewidth]{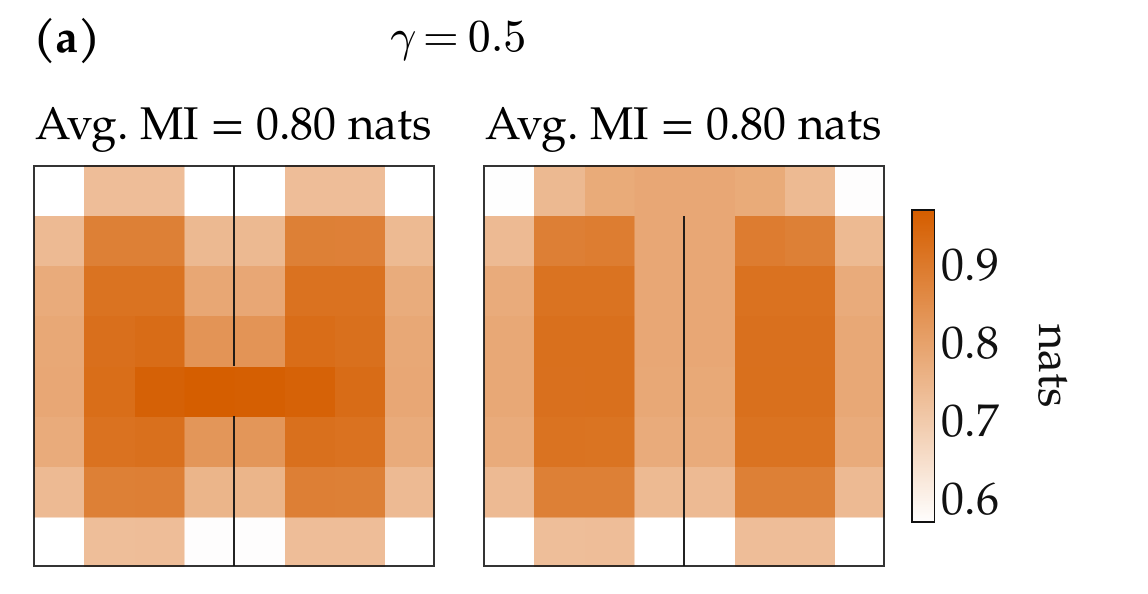}
    \end{minipage}\hfill
    \begin{minipage}[t]{0.48\linewidth}
        \centering
        \includegraphics[width=\linewidth]{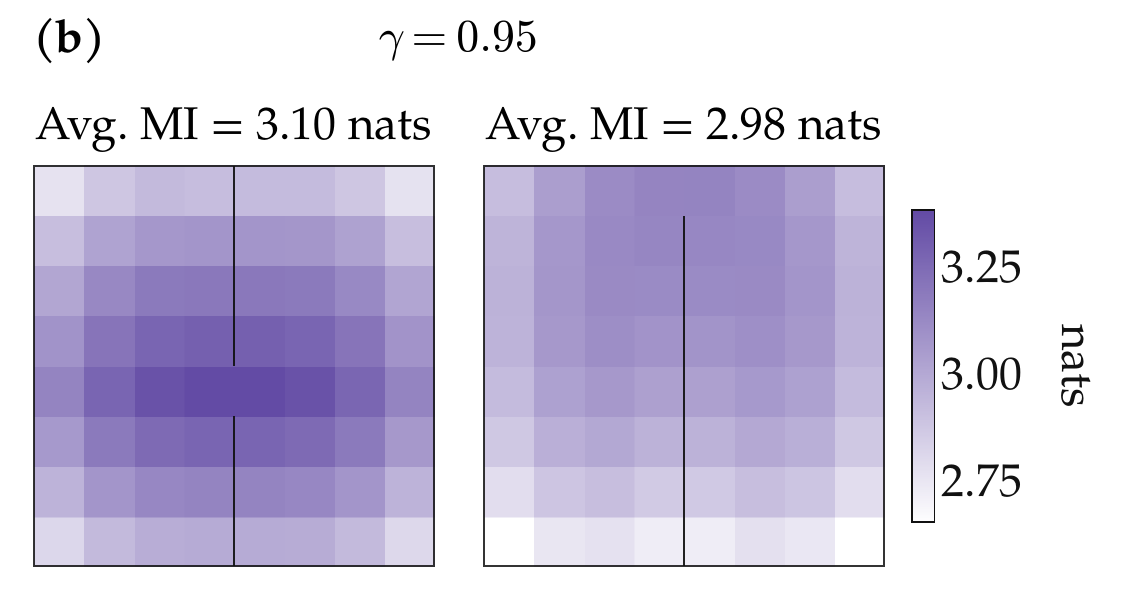}
    \end{minipage}
    \caption{\textbf{Bottleneck states are high potential empowerment states.} Consider GridWorlds with a door, or a ``bottleneck''. (a) For short horizon behaviors where $\gamma = 0.5$, the high empowerment regions are the central door (a, left) and the center of the two sides of the room (a, right). (b) Lifting $\gamma =0.95$ increases the horizon of possible behaviors, such that the empowerment tracks the location of the door. In \cref{subsec:bottlenecks} we show that empowerment measures the reduction in temporal distance to future states, which quantifies a bottleneck effect. See \cref{app:bottleneck_exps} for experimental details.}
    \label{fig:bottleneck}
\end{figure}

\subsection{Is Empowerment Optimal for Downstream Adaptation?}

\label{subsec:downstream}
\paragraph{Potential empowerment lower bounds skill adaptation under an uninformed reward prior.} While empowerment is often introduced as an appealing ``universal''~\citep{klyubin2005all} objective, is there a precise sense in which empowerment measures the ability to adapt in the future?
In \cref{app:proof_adaptation}, we formally prove that empowerment lower bounds skill
adaptation to rewards drawn from an uninformed prior.

To formalize adaptation when rewards are not available during training, we
define an adaptation objective with respect to a reward prior distribution $\cR$.
A simple reward prior is an independent zero-mean Gaussian over each state, with a per-state variance $1/b(s)$ set by a baseline distribution $b$. Here, we take the baseline distribution to be the uniform occupancy.
Prior works use similar isotropic reward distributions when proving general downstream adaptation results~\citep{ollivier2025featuresbestsuccessorfeatures}.
The reward prior treats all states the same and does not prefer particular states.

Here, we define adaptation as a \textit{skill adaptation} to a sampled reward. At test time, an agent must select a skill $z \in \cZ$ to adapt to a sampled reward function $r \sim \cR$. For example, in the case of the $5 \times 5$ grid, an agent should select the skill ``go left'' if the reward provided at test time is on the left.

In the best case, the agent selects the skill $z$ with the highest expected return; note that the agent can do no better within the space of learned skills. This type of policy selection for adaptation/transfer has precedent in prior successor feature literature \citep{barreto2017successor, barreto2018transfer}. For the following result, we consider linear adaptation objective $J$ for the skill set learned by an empowerment-maximizing agent, which is exactly the best skill's mean return averaged over the reward prior. This adaptation objective is notably different from prior results connecting \gls{mi} and adaptation in~\citet{eysenbach2021information}, which considers regularized adaptation directly in the state occupancy. Here, the adaptation is the selection of a \textit{pre-trained} skill behavior, which does not assume the ability to directly adapt state occupancy distributions:

\makerestatable
\begin{theorem}[Tabular empowerment lower bounds reward adaptation]
    \label{thm:width_linear_sharp}
    Consider the start state $s_0$ in state space $\cS$. Let the skill-marginal \gls{dsom} be $p^{\pie}(s_+\mid s_0)=\E_{p(\dmy{z}\mid s_0)}\left[p^{\pie}(s_+\mid s_0, z)\right]$ of \cref{eq:dsom_marginal}. Fix a reference distribution $b$ on $\cS$ with full support (e.g. the uniform baseline $b(s_+)=1/|\cS|$), and sample rewards from the Gaussian
    \[
        r(s_+)\ \sim\ \cN\!\left(0,\ \frac{1}{b(s_+)}\right).
    \]
    Define the linear adaptation objective
    \[
        J(s_0)
        \triangleq
        \E_r\left[\sup_{z\in\cZ}\E_{p^{\pie}(\dmy{s_+}\mid s_0, z)}[r(s_+)]\right],
    \]
    which calculates the average adapted mean reward. Then, for any skill space $\cZ$ and any skill prior $p(z\mid s_0)$,
    \[
        J(s_0)\ \ge\ \frac{I(Z;S_+\mid S_0=s_0)}{c\sqrt{2\pi}},
        \qquad
        c\triangleq\sup_{x\ge0}\frac{\ln(1+x)}{\sqrt{x}}\approx0.8 .
    \]
\end{theorem}

We give the formal statement and proof in \cref{thm:width_linear_sharp} (\cref{app:proof_adaptation}), where $J(s_0)$ turns out to be the \gls{gaussian_width} of the skill offsets $\{u_z\}_{z\in\cZ}$ (\cref{eq:J_width}). Intuitively, for the simple $5 \times 5$ grid (\cref{fig:eff_emp_5x5}), the central maximal potential empowerment state is \textit{exactly} the state from which an agent can easily adapt to the left/right/upper/lower rewarding states in the grid. Thus, maximizing potential empowerment provides a new form of provable adaptation where agents navigate to central states {\it from which} they can later adapt, which existing skill-learning or effective empowerment methods do not capture on their own~\citep{gregor2017variational}.

Two didactic tabular experiments demonstrate the behavior of the bound across \gls{mdp} parameters.
The first set of \gls{mdp}s varies the degree of commitment within an \gls{mdp}, which measures how much probability mass an action places on a target state above uniform.
Commitment measures the controllability of the \gls{mdp}. As controllability increases, the bound tightens (\cref{fig:kl_width_J_quality}).
The second set of \gls{mdp}s varies the number of possible distinct skills/actions while holding commitment fixed.
As the number of possible skills grows, the bound loosens (\cref{fig:kl_width_J_vs_bound}).
The bound is therefore most informative in \gls{mdp}s where an agent exerts strong control over comparatively few distinguishable regions.
Full experimental details are in \cref{app:proof_adaptation}.

\paragraph{The First Limitation of Empowered Adaptation: Gaussian widths versus KL divergence.} While the above result demonstrates that potential empowerment lower bounds adaptation in discrete MDPs, there remain open questions: (1) is the bound tight? and, correspondingly, (2) does the bound extend to continuous state geometries? The answer to both of these questions is generally no: without additional assumptions, the geometry of \textit{reward adaptation} and that of information geometry are different.

First, we must define the geometry of reward adaptation.
\emph{\gls{gaussian_width}} measures the spread of the skills' future occupancies and governs the adaptation geometry~\citep{polyanskiy2022information}.
Because the cumulative reward is \emph{linear} in the occupancy, the width is large when different skills reach futures with clearly different values (\cref{fig:adaptation_geometry}, (b,e)).
This quantity differs from \gls{mi}, where distances are given by the Kullback--Leibler divergence.

\textbf{Our first observation is that the adaptation result is a bound, not a ranking result.}
Consider the simple example of two starting states in a discrete MDP (\cref{fig:adaptation_geometry}, (a) and (d)).
The first starting state in (a) can reach two distinct absorbing states.
The second starting state (d), on the other hand, can reach three states.
Under an isotropic reward prior, an agent optimizing adaptation should choose the second starting state (\cref{fig:adaptation_geometry}, (e)) according to the Gaussian-width geometry that characterizes reward adaptation. Meanwhile, an agent optimizing the information radius will choose the first starting state (\cref{fig:adaptation_geometry}, (c)). This discrepancy is a direct consequence of the difference between information and reward adaptation geometry. Thus, while the lower bound holds, the bound does not guarantee a ranking result with provable counterexamples.

\begin{figure}
    \centering
    \includegraphics[width=0.7\linewidth]{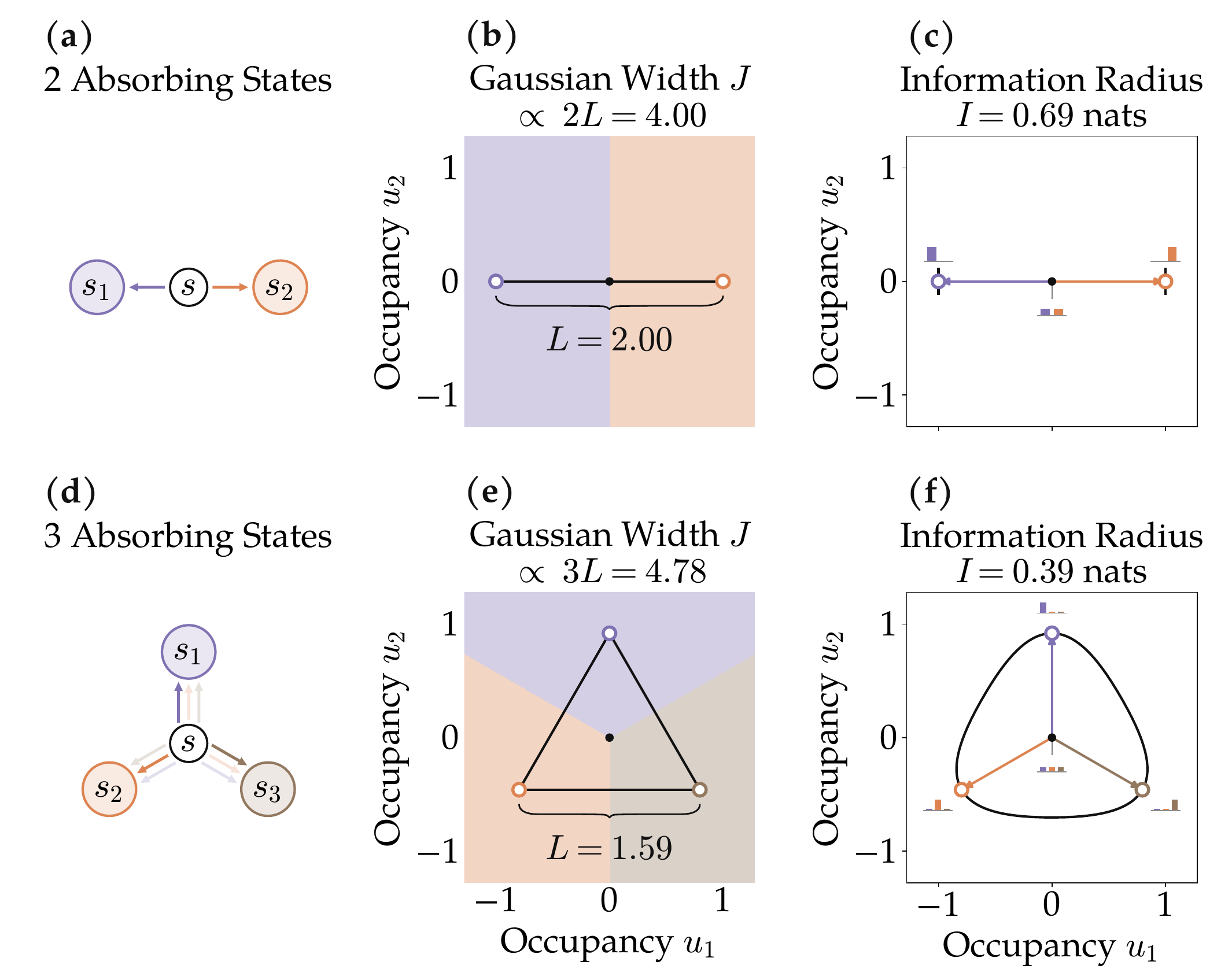}
    \caption{%
        \textbf{Discrepancy between information geometry and reward adaptation geometry in discrete settings.} (a) By going left/right from $s$, an agent can reach node \textcolor{figfivestateone}{$s_1$} or \textcolor{figfivestatetwo}{$s_2$} with high likelihood, so there are two easily separable skills from $s$.
        (d) Targeting nodes \textcolor{figfivestateone}{$s_1$}, \textcolor{figfivestatetwo}{$s_2$}, or \textcolor{figfivestatethree}{$s_3$} can make the agent inadvertently land in other nodes, so there are three somewhat separable skills from $s$.
        (b,e) The \gls{mdp}s in (a) and (d) admit polytopes of feasible state occupancies in the simplex, visualized by the segment and triangle.
        (b) Axis $u_1$ trades off occupancy of \textcolor{figfivestateone}{$s_1$} vs.
        \textcolor{figfivestatetwo}{$s_2$}.
        (e) Axis $u_1$ trades off occupying \textcolor{figfivestatetwo}{$s_2$} vs.
        \textcolor{figfivestatethree}{$s_3$}; axis $u_2$ trades off occupying \textcolor{figfivestateone}{$s_1$} vs.
        $\textcolor{figfivestatetwo}{s_2} / \textcolor{figfivestatethree}{s_3}$.
        Return is linear in the state occupancy, so we can also represent reward functions as vectors in state-occupancy space corresponding to the reward weights.
        The color indicates the best adapted skill for that reward coordinate.
        The polytope in (e) has a larger Gaussian Width than that in (b), supporting better adaptation.
        (c,f) However, increased skill separability in (a) gives larger information radius and \gls{mi}, where the black contour indicates the maximal \gls{mi}.
        Thus, even though empowerment is higher for the 2 state system (a) versus the 3 state system (d), the 2 state system has worse adapted returns.
        See \cref{app:adaptation_geometry_exps} for details.%
    }
    \label{fig:adaptation_geometry}
\end{figure}

In the continuous state limit, the difference between empowerment and adaptation geometry becomes more salient:
\begin{theorem}[Informal statement that large \gls{mi} does not necessarily imply adaptation to smooth rewards in continuous state \gls{mdp}s]
    For every skill count $K$ and $\epsilon$, there exists a continuous state MDP admitting $K$ distinct skills that induce \gls{dsom}s within an $\epsilon$-ball in arbitrary continuous state space.

    Taking $K \to \infty$, the \gls{mdp} admits a set of skills that induce unbounded empowerment while adaptation to a Lipschitz-continuous reward can be arbitrarily small.
\end{theorem}

We give the formal statements and proofs in \cref{prop:admits_all_k,prop:mi_lipschitz_collapse}. Our result shows that skills can have unbounded distinguishability but fail to adapt to smooth rewards, inducing identical returns across the skill set \cref{fig:mi_ball_collapse} illustrates how information geometries can lead to collapsed returns in a continuous 2D state space, where the potential empowerment grows from left to right while returns homogenize. Thus, larger potential empowerment does not imply better adaptation to smooth rewards in continuous state \gls{mdp}s. We fully discuss this collapse in \cref{app:proof_continuous_collapse}.
\begin{figure}
    \centering
    \includegraphics[width=0.7\linewidth]{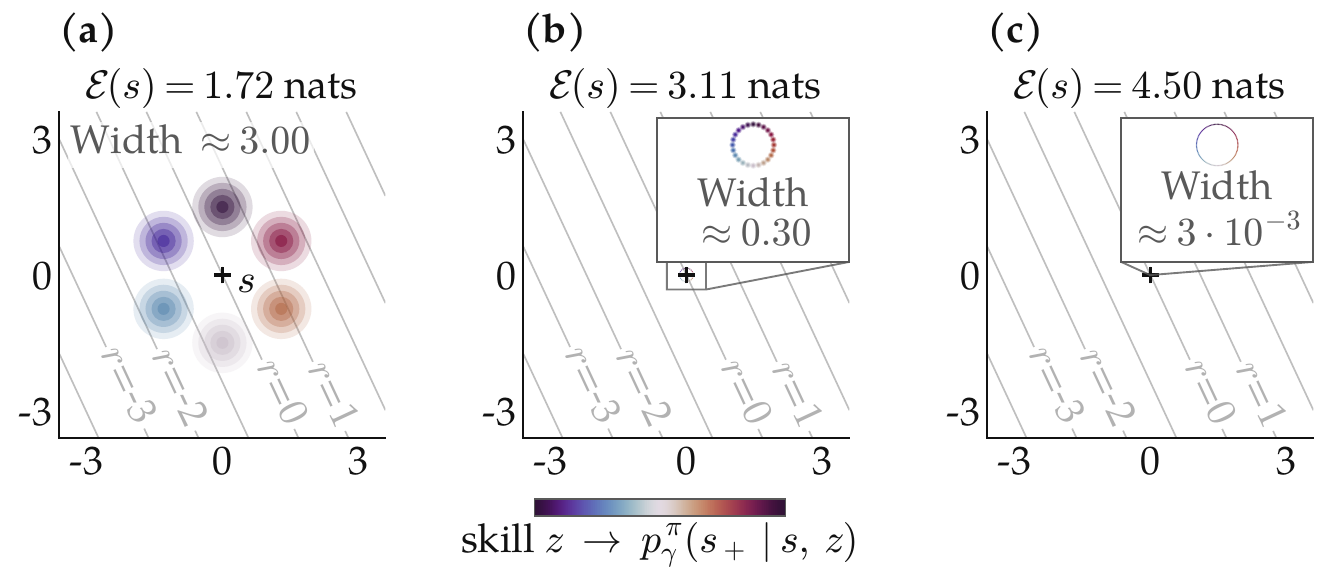}
    \caption{\textbf{The information ball can collapse in the continuous limit} (\cref{prop:mi_lipschitz_collapse})\textbf{.} Consider an arbitrary smooth reward function (level curves in gray) and some \gls{mdp} that admits Gaussian skill state occupancy measures (colored distributions). While the potential empowerment of the starting state $s \in \cS$ increases from subfigure (a) $\rightarrow$ (c), the skills have increasingly identical returns and, thus, provide progressively smaller adaptation benefit. The plotted width is the maximum distance between skill means, not to be confused with Gaussian Width. This width bounds the differences in expected return for the smooth reward and decreases from left to right. Therefore, large channel capacities do not imply adaptation for continuous \gls{mdp}s and reward functions. See \cref{prop:mi_lipschitz_collapse} for the theoretical result.}
    \label{fig:mi_ball_collapse}
\end{figure}

\paragraph{The Second Limitation of Empowered Adaptation: Controllable Degrees of Freedom versus Rewarding Degrees of Freedom.} The second limitation of empowered adaptation is the misalignment between controllable degrees of freedom and preferred degrees of freedom. While the adaptation result of \cref{thm:width_linear_sharp} requires an isotropic reward prior, naive empowerment maximization cannot generically provide adaptation guarantees for anisotropic reward priors. Without external information, empowerment can only break symmetry across states via controllability.
Then, the reward prior is free to vary in directions in the state occupancy simplex that empowerment does not naturally select:

\begin{theorem}[Informal statement that large \gls{mi} need not imply adaptation to anisotropic reward function priors in discrete settings]
    A state can have arbitrarily large \gls{mi} while adaptation with respect to an anisotropic reward prior is arbitrarily close to zero.
\end{theorem}
The formal statements and proofs are in \cref{app:anisotropic} (\cref{prop:anisotropic_degraded,prop:anisotropic_collapse}). We can reinterpret \cref{fig:adaptation_geometry} as an example where the maximizing \gls{mi} selects for skills that do not control degrees of freedom relevant to the anisotropic return: if the variation in an anisotropic reward prior concentrates on three absorbing states in Fig. 5(d), an empowerment-maximizing agent will ignore those reward-relevant states and, instead, navigate towards the state with maximal channel capacity. Therefore, the skill set is suboptimal for adaptation to the anisotropic reward prior.

\section{Discussion, Limitations, \& Future Work}
\label{sec:disc}
In this work, we analyze the geometry of empowerment. Our theoretical results show that high potential empowerment states are central states in information geometry (\cref{thm:effective_info_geom}). The empowerment objective and a multistep skill rollout variant also admit exact hitting time interpretations and capture centrality in a temporal distance geometry (\cref{lem:bottleneck,prop:rollout_temporal_distance}). Thus, policies that maximize potential empowerment naturally seek out central states within an \gls{mdp}.  Furthermore, effective empowerment policies, or skills, measure the degree of centrality (\cref{fig:eff_emp_5x5}). Finally, in the tabular setting, this notion of centrality achieved by a set of skills directly lower bounds an agent's ability to adapt to unknown rewards (sampled from an uninformed prior). Our results formally link empowerment and the ability to complete downstream tasks (\cref{thm:width_linear_sharp}).

\paragraph{Limitations \& Future Work.} That being said, while the above results provide support for empowerment as a generic objective, information geometry and return geometries are fundamentally different: reward adaptation geometry is typically that of Gaussian widths, while information geometry is that of KL-divergences. This observation has a few downstream implications. First, a skill set that obtains a larger \gls{mi} does not necessarily exhibit better adaptation; our adaptation result is a lower bound, not a ranking result (\cref{fig:adaptation_geometry}). Second, large empowerment does not imply adaptation in continuous \gls{mdp}s with Lipschitz-continuous rewards: there exist \gls{mdp}s where skill sets can obtain unbounded empowerment, while providing no adaptation benefit to a continuous reward function (\cref{prop:mi_lipschitz_collapse,fig:mi_ball_collapse}). Third, empowerment \emph{alone} does not provide theoretical guarantees on downstream adaptation to anisotropic reward priors (\cref{prop:anisotropic_degraded}) that prefer certain states over others.

Thus, while empowerment has attractive properties in addition to a straightforward information-theoretic interpretation, there exist theoretical gaps between the empowerment object and downstream task adaptation. Our work provides a partial answer to Salge's hypothesis~\citep{salge2014empowerment}:
while empowerment, a task-agnostic quantity, does provably lower bound goal-oriented behavior in tabular settings, this hypothesis is not generally true for current formulations of empowerment maximization and return maximization. Possible future directions include developing scalable formulations of continuous empowerment maximization studied in prior works~\citep{jung2011empowerment, tiomkin2017control, zhao2021efficientempowermentestimationunsupervised} and modifying recent scalable skill-learning algorithms for empowerment estimation ~\citep{levy2025representation, zheng2025misl, park2023metra}.

Until we consider exploration and adaptation methods that appropriately unify these geometries, there will remain a disconnect between such intrinsic reward approaches and the broader RL problem setting. Indeed, our analysis can generically apply to any method that combines extrinsic and information-theoretic intrinsic rewards~\citep{houthooft2017vime}. We believe that scalably addressing these limitations is key in order to make information-based intrinsic rewards, including empowerment, relevant to the broader RL community and is the natural next step for future work.

\paragraph{Acknowledgments.} We would like to thank David Abel, André Barreto, Mahsa Bastankhah, Raphaël Bauer, Max Kleiman-Weiner, Andrew Levy, Sikata Sangupta, and anonymous reviewers for their insight and helpful feedback. This material is based upon work supported by Google and the National Science Foundation Graduate Research Fellowship Program under Grant No. DGE-2444107 and Award No. 2441665. Any opinions, findings, and conclusions or recommendations expressed in this material are those of the author(s) and do not necessarily reflect the views of the National Science Foundation. We credit Font Awesome (Fonticons, Inc.) for the icons in \cref{fig:eff_emp_5x5,fig:example-center,fig:pot_emp}, adapted under \href{https://creativecommons.org/licenses/by/4.0/}{CC BY 4.0}.

\paragraph{Generative AI Disclosure Statement.} We acknowledge the use of generative AI tools for figure generation, setting up the paper website, code implementation for the tabular experiments, and minor edits to improve readability.

\appendix
\onecolumn
\printglossaries

\section{Preliminaries for Proofs}
\label{app:proofs}

We restate four standard results in forms used in our statements. We cite referenced restatements of these results, which are more general versions of the results below.

The first result gives the standard concavity and convexity properties
of mutual information
\citep[Thm.~2.7.4]{cover2006elements}
\citep[Thm.~5.3]{polyanskiy2022information}.

\begin{theorem}[Concavity and convexity of mutual information]
    \label{thm:mi_convexity}
    Let $P_X$ be a source distribution and let $P_{Y\mid X}$ be a channel
    between standard Borel spaces. For a fixed channel,
    $I(X;Y)$ is concave in $P_X$. For a fixed source, $I(X;Y)$ is convex
    in $P_{Y\mid X}$.
\end{theorem}

In our setting, the source is the skill prior
$p(z\mid s_0)$ and the channel maps each skill $z$ to a \gls{dsom}
$p^\pi(\dmy{s_+}\mid s_0,z)$. Thus, the \gls{mi}
$I^\pi(Z;S_+\mid S_0=s_0)$ is concave in the skill prior for fixed
skill occupancies and convex in the occupancies for a fixed skill
prior.

The second result expresses the channel capacity in information
radius form as the smallest worst-case divergence from the channel outputs to a central distribution
\citep[Sec.~5.3]{polyanskiy2022information}.

\begin{theorem}[Dual form of channel capacity]
    \label{thm:capacity_information_radius}
    Let $P_{Y\mid X}$ be a channel between standard Borel spaces $\cX$ and
    $\cY$, and write $P_{Y\mid X=x}$ for the output distribution at input
    $x\in\cX$. Then
    \[
        \sup_{P_X\in\Delta(\cX)} I(X;Y)
        =
        \inf_{q(\dmy y)}
        \sup_{x\in\cX}
        \KL\bigl(P_{Y\mid X=x}(\dmy y)\,\|\,q(\dmy y)\bigr).
    \]
    If this value is finite, the infimum is attained at a unique center
    $q^\star$.
\end{theorem}

In our setting, the input is the skill, $X=Z$, and the output is the
future state, $Y=S_+$, so each output distribution $P_{Y\mid X=x}$ is a
skill occupancy $p^\pi(\dmy{s_+}\mid s_0,z)$. We utilize this dual form to prove
\cref{thm:effective_info_geom_general}.

The third result compares KL and $\chi^2$ divergence
\citep[Eq.~(7.31)]{polyanskiy2022information};
see also \citet{gibbs2002choosing} for the general statement.
\begin{theorem}[KL and $\chi^2$ comparison]
    \label{thm:kl_chi2}
    For any distributions $P\ll Q$,
    \[
        \KL(P\|Q)
        \le
        \log\bigl(1+\chi^2(P\|Q)\bigr).
    \]
\end{theorem}

We apply this bound with
\[
    P=p^{\pie}(\dmy{z}\mid s_0,s_+)
    \quad\text{and}\quad
    Q=p(\dmy{z}\mid s_0).
\]
Averaging the bound over $s_+$ gives the
$\chi^2$ control of mutual information that appears in
\cref{thm:width_linear_sharp}.

The fourth result shows that a measurable transformation of noise can induce any distribution on a standard Borel space
\citep[Lem.~5]{bloemreddy2020probabilistic}
\citep[Thm.~8.17]{kallenberg2021foundations}.

\begin{theorem}[Noise outsourcing]
    \label{thm:transfer}
    Let $\mu$ be a probability distribution on a standard Borel space
    $\cZ$. Let $U$ be the uniform distribution on a
    standard Borel space $\cV$. If $\tilde Z\sim U$, then there exists a
    measurable map $T:\cV\to\cZ$ such that
    \[
        T(\tilde Z)\sim\mu.
    \]
\end{theorem}

We utilize this statement to prove that a uniform skill prior does not lose generality in \cref{lem:uniform_continuous}.

\section{Information Geometry of Empowerment}
\label{app:proof_info_geom}

We prove a general theorem of the information radius result from \cref{subsec:geom} that applies to both continuous and discrete state space \gls{mdp}s. We include the reduction of the result to the finite-state statement found in \cref{subsec:geom}.

We start with intuition about the information channel, then explain potential simplications of this information channel. First, note that skill latents $z \in \cZ$ can only influence future state distributions via a skill-conditioned policy. The information channel $P_{S_+ \mid Z}$ is a function of the skill-conditioned policy $\pi \in \Pi_\cZ$, which sends each skill to the induced occupancy:
\begin{align*}
    \cZ \xrightarrow{\ \pi\ } \cP(s_0),
    \qquad
    z \longmapsto p^\pi(\dmy{s_+} \mid s_0, z).
\end{align*}
If this mapping is measurable, we can directly reduce the skills to their induced \gls{dsom}s. Specifically, there would exist a pushforward map from the skill random variable $Z$ to a random state occupancy $P$ drawn from a \emph{distribution over} $\cP(s_0)$ (e.g. a distribution over state occupancy distributions) which, by the data-processing inequality from the Markov chain  $Z \rightarrow P_Z \rightarrow S_+$ and the sufficiency of $P$, would imply
\begin{align*}
    I(Z;S_+ \mid S_0) = I(P_Z;S_+ \mid S_0).
\end{align*}
Note, however, that not all policies admit a measurable mapping from skill space to occupancy space. We require the additional assumption that the stochastic policies $\pi(a \mid s, z)$ are jointly measurable over the state and skill space:
\begin{assumption}[Policies measurable in skill and state space]
    \label{asm:selector}
    Let $\cS$, $\cA$, and $\cZ$ denote the state space, action space, and skill spaces respectively, which we take to be standard Borel.

    Take the skill-conditioned policy space $\Pi_\cZ$ to be the space of policies measurable in the state space $\cS$ and skill space $\cZ$. That is,
    \[
        (s,z) \mapsto \pi(\cdot \mid s,z) \in \Delta(\cA)
    \]
    is a measurable function.
\end{assumption}
Accordingly, the state occupancy set $\cP(s_0)$ is now  occupancies reachable by measurable policies.

From this assumption, there exists a measurable mapping from $\cZ \xrightarrow{\ \pi\ } \cP(s_0)$. Thus, we can effectively alias the skill latent $Z$ and the random state occupancy measure $P_Z$. For the information geometry result, this association allows direct application of \cref{thm:capacity_information_radius} to the information channel from skills to occupancies:

\begin{theorem}[General information geometry of empowerment]
    \label{thm:effective_info_geom_general}
    Fix a start state $s_0$. Then, with both sides taking values in
    $[0,\infty]$,
    \[
        \gEp(s_0)
        =
        \inf_{q(\dmy{s_+})}
        \sup_{p\in\cP(s_0)}
        \KL\bigl(p(\dmy{s_+})\|q(\dmy{s_+})\bigr).
    \]
    If the empowerment is bounded, the infimum is attainable. Let the skill-conditioned policy $\pi \in \Pi_\cZ$ and the skill source $p(z\mid s_0)$ over skills $z \in \cZ$ attain $\gEp(s_0)$. Then, the skill marginal is the unique center: \[
        q^\star=p^\pi(\dmy{s_+}\mid s_0).
    \]
    Moreover, for $p(z\mid s_0)$-almost every $z$,
    \[
        \KL\bigl(
        p^\pi(\dmy{s_+}\mid s_0,z)
        \|q^\star(\dmy{s_+})
        \bigr)
        =
        \mathbb{E}
        \gEp(s_0),
    \]
    where $p^\pi(\dmy{s_+}\mid s_0,z)$ is an extreme point of
    $\cP(s_0)$. Thus, an optimal policy places every skill occupancy at the maximal KL radius, while an optimal source induces the average occupancy $q^\star$.
\end{theorem}

We now prove \cref{thm:effective_info_geom_general} with the
information-radius identity of \cref{thm:capacity_information_radius}.
\begin{proof}
    Fix $S_0=s_0$. Each skill $z$ induces a future state occupancy
    \[
        p^\pi(\dmy{s_+}\mid s_0,z)\in\cP(s_0),
    \]
    and the conditional distribution of $S_+$ depends on $z$ only through
    this occupancy. From the measurability assumption \cref{asm:selector}, we may alias skills that induce the same
    occupancy. Thus, jointly optimizing the
    policy and skill source is equivalent to choosing a source over
    $\cP(s_0)$ for the channel
    \[
        p\in\cP(s_0),
        \quad
        S_+\sim p.
    \]
    In other words, the source is a distribution \emph{over a set of state-occupancy distributions} that give the channel output $S_+$.

    Applying \cref{thm:capacity_information_radius} with the random
    occupancy $P_Z\in\cP(s_0)$ as the input and the future state $S_+$ as the output gives
    \[
        \gEp(s_0)
        =
        \inf_{q(\dmy{s_+})}
        \sup_{p\in\cP(s_0)}
        \KL\bigl(p(\dmy{s_+})\|q(\dmy{s_+})\bigr).
    \]
    Thus, assuming the potential empowerment is finite, a unique center $q^\star$ attains the infimum. Now, suppose that the skill policy and the prior attains the maximal capacity
    $C=\gEp(s_0)$. For any distribution $q$, we have that
    \[
        \E_{p(\dmy{z}\mid s_0)}
        \left[
            \KL\bigl(
            p^\pi(\dmy{s_+}\mid s_0,z)
            \|q(\dmy{s_+})
            \bigr)
            \right]
        =
        I^\pi(Z;S_+\mid S_0=s_0)
        +
        \KL\bigl(
        p^\pi(\dmy{s_+}\mid s_0)
        \|q(\dmy{s_+})
        \bigr)
    \]
    following the Golden Formula (\citealp[Thm.~4.1]{polyanskiy2022information}). Taking $q=q^\star$, every reachable occupancy has divergence at most
    $C$.
    Hence,
    \begin{align}
        C
         & \ge
        \E_{p(\dmy{z}\mid s_0)}
        \left[
            \KL\bigl(
            p^\pi(\dmy{s_+}\mid s_0,z)
            \|q^\star(\dmy{s_+})
            \bigr)
        \right]                     \\
         & =
        C
        +
        \KL\bigl(
        p^\pi(\dmy{s_+}\mid s_0)
        \|q^\star(\dmy{s_+})
        \bigr) \tag{Golden formula} \\
         & \ge C.
    \end{align}
    Thus, the KL between the marginal state occupancy and $q^\star$ is $0$, which implies
    \[
        q^\star=p^\pi(\dmy{s_+}\mid s_0).
    \]
    and that every individual skill divergence is
    at most $C$:
    \begin{align}
        \KL\bigl(
        p^\pi(\dmy{s_+}\mid s_0,z)
        \|q^\star(\dmy{s_+})
        \bigr)
        =
        C
        \label{eq:capacity_achieving}
    \end{align}
    for $p(z\mid s_0)$-almost every $z$. Thus, an optimal source places skill occupancy
    mass on the boundary of the optimal KL ball, with the KL measured from the central marginal distribution $q^\star$.

    Finally, suppose there exists an active skill occupancy that is not an extreme point of
    $\cP(s_0)$. Then, for distinct $p_1,p_2\in\cP(s_0)$ and some
    $\lambda\in(0,1)$,
    \[
        p^\pi(\dmy{s_+}\mid s_0,z)
        =
        \lambda p_1+(1-\lambda)p_2.
    \]
    Convexity of the KL divergence in the first argument gives
    \[
        \begin{aligned}
            \KL\bigl(
            p^\pi(\dmy{s_+}\mid s_0,z)
            \|q^\star(\dmy{s_+})
            \bigr)
             & <
            \lambda\KL\bigl(p_1(\dmy{s_+})\|q^\star(\dmy{s_+})\bigr)
            +(1-\lambda)\KL\bigl(p_2(\dmy{s_+})\|q^\star(\dmy{s_+})\bigr) \\
             & \le C,
        \end{aligned}
    \]
    which contradicts the equality in \cref{eq:capacity_achieving}. Thus,
    $p(z\mid s_0)$-almost every skill induces an extreme point of
    $\cP(s_0)$.
\end{proof}

The reduction to the finite tabular case directly follows:

\restatetheorem{thm:effective_info_geom}

\begin{proof}[Proof of \cref{thm:effective_info_geom}]
    In the finite-state setting, $\cP(s_0)$ is a polytope and
    $\gEp(s_0)\le\log|\cS|$, so the channel capacity is finite.
    \cref{thm:effective_info_geom_general} shows that the center distribution is the
    skill-marginal where active skill occupancies lie on the boundary of
    the optimal KL ball at extreme points of $\cP(s_0)$. The extreme
    points of a polytope are vertices, which proves the result.
\end{proof}

\subsection{Log reachability interpretation of potential empowerment in discrete deterministic settings}
For completeness, we include the log reachability interpretation of the action sequence potential empowerment (as originally defined in~\citep{klyubin2005empowerment}) in discrete deterministic settings. \citet{salge2014empowerment} originally present this result.

We begin by defining the action sequence potential empowerment. Let $A_0^{H-1}=(A_0,\ldots,A_{H-1})\in\cA^H$ be an open-loop action sequence over the fixed horizon $H$. Let $S_t$ denote the random variable of the state after $t$ timesteps. \citet{klyubin2005empowerment} define (potential) empowerment as the capacity of the channel from this action sequence to the resulting state:
\begin{align*}
    \gEp_H(s)=\max_{p(a_0^{H-1})} I(A_0^{H-1};S_H\mid S_0=s).
\end{align*}
This expression is equivalent to \cref{def:pot}, replacing $Z$ with the action sequence $A_0^{H-1}$ and $S_+$ with $S_H$.

\begin{lemma}[Discrete deterministic potential empowerment is log reachability, adapted from \citet{salge2014empowerment}]
    \label{lem:deterministic_reachability}
    Fix a discrete, deterministic MDP and horizon $H$.
    Because dynamics are deterministic, the future state $S_H$ is a deterministic function of $S_0$ and $A_0^{H-1}$.
    Then, the reachable set of states in $H$ timesteps takes the form
    \[
        \cR_H(s)=\{s_+\in\cS:\ \exists\, a_0^{H-1}\in\cA^H\text{ that reaches }S_H=s_+ \text{ from } S_0=s\},
    \]
    and the empowerment at $s$ is the log of the size of this set:
    \[
        \gEp_H(s)=\log |\cR_H(s)|.
    \]
\end{lemma}

\begin{proof}
    Because $S_H$ is a deterministic function of $A_0^{H-1}$ given $S_0=s$, the conditional entropy $H(S_H\mid A_0^{H-1},S_0=s)$ is $0$. Then,
    \[
        I(A_0^{H-1};S_H\mid S_0=s)=H(S_H\mid S_0=s).
    \]
    Because the support of $S_H$ is in the reachable set $\cR_H(s)$, the entropy is at most $\log|\cR_H(s)|$. By construction, we can choose one action sequence for each reachable $s_+\in\cR_H(s)$ and place uniform mass $p(a_0^{H-1})$ over the chosen action sequences. Thus, the random variable $S_H$ is uniform on $\cR_H(s)$ with entropy $\log|\cR_H(s)|$.
\end{proof}

\section{Empowerment and Bottleneck States}
\label{app:bottleneck}

We state and prove the hitting-time and temporal-distance statements from \cref{subsec:bottlenecks}.
\subsection{Bottleneck states are high-empowerment states}

We first define the repeated rollout hitting time from the
\gls{dsom}. A single rollout from $s_0$ samples a skill
$z\sim p(\dmy z\mid s_0)$ and a geometric termination time
$\tau\sim\mathrm{Geom}(1-\gamma)$. The agent
executes skill $z$ until $\tau$, where the rollout terminates at
$S_+\triangleq S_\tau$.

Writing $S_+^{(k)}$ for the
termination state of the $k$th rollout from $s_0$, the
repeated-rollout hitting time measures the first rollout count that hits $s_+$:
\[
    H^\pi(s_0\to s_+)
    \triangleq
    \inf
    \left\{
        k\ge 1:
        S_+^{(k)}=s_+
    \right\},
\]
where $H^\pi(s_0\to s_+\mid z)$ is the corresponding count when
every rollout only uses skill $z$. Then, potential empowerment is exactly the largest expected reduction in log
hitting time after committing to a skill:
\lemBottleneck*

\begin{proof}
    The repeated-rollout hitting times are geometric random variables
    with success probabilities $p^\pi(s_+\mid s_0)$ and
    $p^\pi(s_+\mid s_0,z)$, respectively. Hence,
    \[
        \E[
            H^\pi(s_0\to s_+)
        ]
        =
        \frac{1}{p^\pi(s_+\mid s_0)},
        \quad
        \E[
            H^\pi(s_0\to s_+\mid z)
        ]
        =
        \frac{1}{p^\pi(s_+\mid s_0,z)}.
    \]
    The \gls{mi} is the expected log ratio of the skill-conditioned to
    the marginal \gls{dsom}, so
    \begin{align}
        I^\pi(Z;S_+\mid S_0=s_0)
        &=
        \E_{p^\pi(z,s_+\mid s_0)}
        \log
        \frac{
            p^\pi(s_+\mid s_0,z)
        }{
            p^\pi(s_+\mid s_0)
        }
        \notag\\
        &=
        \E_{p^\pi(z,s_+\mid s_0)}
        \left[
            \log
            \E[
                H^\pi(s_0\to s_+)
            ]
            -
            \log
            \E[
                H^\pi(s_0\to s_+\mid z)
            ]
        \right].
        \label{eq:mi_hitting_gap}
    \end{align}
    This identity holds for any policy. Maximizing both sides over
    effective empowerment policies $\pie$ gives the result.
\end{proof}

\subsection{Temporal distances over contiguous skill rollouts}
\label{app:rollout_temporal}

The above result considers the empowerment objective used throughout the rest of the
paper, which measures the controllability of a single skill on the termination state/discounted future observation.
We now define a closely related objective over \emph{contiguous} rollouts where the temporal distance interpretation from \citet{myers2025learning}
is exact.

Rather selecting a single skill and rolling out until termination, the agent can execute contiguous skill rollouts in a single trajectory. An agent can begin by selecting a skill $z$. At some geometric termination time $\tau \sim \text{Geom}(1-\gamma)$, the agent resamples a random skill and continues forth from the terminated state. This process repeats forever. Intuitively, this rollout scheme is a generalization of standard \gls{mdp} dynamics, where skill $Z$ now takes the role of the action $A$, and the $\beta$ discount factor over units of rollouts controls the fundamental timescale, in addition to the standard $\gamma$.

We can mathematically write out the state occupancy measure of such a process. Let the random variable $X_k$ be the state occupied at the $k$-th geometric termination time. Let the trajectory begin at state $X_0 = s_0$. At the $k$-th termination, the agent re-draws a skill $Z_k\sim p(\dmy z\mid X_k)$ and continues the trajectory. Then, the state occupancy of the process takes the form
\[
    M_\beta^{\pi}(g\mid s)
    \triangleq
    (1-\beta)
    \sum_{k=0}^{\infty}
    \beta^k
    p^{\pi}(X_k=g\mid X_0=s).
\]
We can correspondingly define the state and skill-conditioned \gls{dsom} of the contiguous process as
\[
    M_\beta^{\pi}(g\mid s, z)
    \triangleq
    (1-\beta)
    \sum_{k=0}^{\infty}
    \beta^k
    p^{\pi}(X_k=g\mid X_0=s, Z_0 = z),
\]
where the agent rolls out skill $z$ until the first geometric time, then resamples skills from the prior $p(z \mid s_k)$ thereafter. Note that averaging the skill-conditioned \gls{dsom} gives the state-conditioned \gls{dsom}:
\begin{equation}
    \E_{z\sim p(\dmy z\mid s)}
    \left[
        M_\beta^{\pi}(g\mid s,z)
        \right]
    =
    M_\beta^{\pi}(g\mid s).
    \label{eq:rollout_successor_consistency}
\end{equation}
The \textbf{effective contiguous empowerment} measures the degree to which the first skill controls the occupancy of the entire process. Let $Z\sim p(\dmy{z}\mid s_0)$ be the first skill and let $G\sim M_\beta^{\pi}(\dmy{g}\mid s_0,Z)$ be the state the process occupies at a geometrically sampled termination time. Taking the supremum over the policy and source gives the channel capacity form, or the \textbf{potential contiguous empowerment}:
\begin{align*}
    \gE^{\mathrm{contig}}_{\mathrm{eff}}(\pi,s_0)
     & \triangleq
    I^{\pi}(Z;G\mid X_0=s_0),
    \\
    \gE^{\mathrm{contig}}_{\mathrm{pot}}(s_0)
     & \triangleq
    \sup_{\pi \in \Pi_\cZ,\, p(z\mid s)}
    \gE^{\mathrm{contig}}_{\mathrm{eff}}(\pi,s_0).
\end{align*}
An agent is effectively contiguously empowered if their intial skill/generalized action exerts control on their future state for the contiguous process. An agent is potentially contiguously empowered when occupying a state from which immediate skills provide long-term control over the future. We highlight that this modification of empowerment is conceptually similar to the definitions of empowerment considered in the rest of the paper, despite different mathematical forms.

The potential contiguous empowerment admits a temporal distance interpretation, following \citet{myers2025learning}. The temporal distance captures a notion of the time to traverse between states that generalizes to continuous state settings. Let $\beta$ be the discount factor. Let $M^\pi_\beta (g \mid s)$ denote the \gls{dsom} of state $g$ for a process that begins at $s$. Then, the temporal distance from a state $s$ to a future state $g$ is defined as
\[
    d^\pi_\beta (s,g) \triangleq -\log \frac{M_\beta^{\pi}(g\mid s)}{M_\beta^{\pi}(g\mid g)}
\]
where the denominator measures the self-reachability probability for a process that begins at $g$. Conditioning on the initial skill gives the skill-conditioned temporal distance
\[
    d^\pi_\beta (s,g \mid z) \triangleq -\log \frac{M_\beta^{\pi}(g\mid s, z)}{M_\beta^{\pi}(g\mid g)},
\]
again normalized by the self-reaching probability. Let
$M_\beta^{\pi}(\dmy{z},\dmy{g}\mid s_0)= p(\dmy{z}\mid s_0)\,M_\beta^{\pi}(\dmy{g}\mid s_0,z)$
be the joint law of the initial skill $Z$ and future state
$G$. Adding and subtracting the self-reaching term unifies the contiguous empowerment and the temporal distance:

\begin{proposition}[Contiguous empowerment measures temporal distance reduction]
    \label{prop:rollout_temporal_distance}
    For any skill policy $\pi \in \Pi_\cZ$ and skill source $p(z\mid s)$,
    \[
        \gE^{\mathrm{contig}}_{\mathrm{eff}}(\pi,s_0)
        =
        \E_{M_\beta^{\pi}(\dmy{z},\dmy{g}\mid s_0)}
        \left[
            d_\beta^{\pi}(s_0,g)
            -
            d_\beta^{\pi}(s_0,g\mid z)
            \right].
    \]
    Consequently,
    \[
        \gE^{\mathrm{contig}}_{\mathrm{pot}}(s_0)
        =
        \sup_{\pi \in \Pi_\cZ,\, p(z\mid s)}
        \E_{M_\beta^{\pi}(\dmy{z},\dmy{g}\mid s_0)}
        \left[
            d_\beta^{\pi}(s_0,g)
            -
            d_\beta^{\pi}(s_0,g\mid z)
            \right].
    \]
    Thus, the potential contiguous empowerment is the maximum expected reduction in temporal distance to future states when committing to an initial skill $Z$.
\end{proposition}

\begin{proof}
    The result follows directly from definitions. The self-reaching term cancels:
    \[
        d_\beta^{\pi}(s_0,g)
        -
        d_\beta^{\pi}(s_0,g\mid z)
        =
        \log
        \frac{
            M_\beta^{\pi}(g\mid s_0,z)
        }{
            M_\beta^{\pi}(g\mid s_0)
        }.
    \]
    Taking the expectation over $M_\beta^{\pi}(\dmy{z},\dmy{g}\mid s_0)$ gives the equality between the temporal distance reduction and effective contiguous empowerment. Taking the supremum of both sides over the skill-conditioned policy and source gives the potential contiguous empowerment expression.
\end{proof}

Unlike the hitting time expression, this form of the contiguous potential empowerment holds in continuous settings.
In discrete spaces, the temporal distance reduces to a typical hitting time
interpretation \citep{myers2025learning}. Let
\[
    H_{\mathrm{contig}}^{\pi}(s\to g)
    \triangleq
    \inf\{k\ge 0:X_k=g\}
\]
be the time to hit $g$ from $s$, counted in units of contiguous skill rollouts. Let $H_{\mathrm{contig}}^{\pi}(s\to g\mid z)$ be the corresponding skill-conditioned hitting time. Following \citet{myers2025learning}, we can rewrite the temporal distance in terms of these hitting times:
\[
    d_\beta^{\pi}(s,g)
    =
    -\log
    \E\left[
        \beta^{
            H_{\mathrm{contig}}^{\pi}(s\to g)
        }
        \right],
    \quad
    d_\beta^{\pi}(s,g\mid z)
    =
    -\log
    \E\left[
        \beta^{
            H_{\mathrm{contig}}^{\pi}(s\to g\mid z)
        }
        \right]
    \quad
    \text{for } g\neq s.
\]
Thus, there is a clear hitting time interpretation of the temporal distance object in the discrete limit.

\paragraph{Note on different empowerment expressions.} The contiguous empowerment object differs from the empowerment objective used elsewhere in
the paper. All other statements in the paper
characterize the original
potential empowerment; this result is an extension that provides a neat connection to temporal distance as defined in \citep{myers2025learning}, which (unlike hitting time) remains well-defined in continuous settings.Estimating the contiguous empowerment in tabular settings requires a methodology different from estimating empowerment considered elsewhere in the paper, as the channel mapping from skills to outcomes is also a function of the source $p(z \mid s)$. In addition to doing the randomized vertex search, we estimate contiguous empowerment by doing gradient descent on a parameterized source. Further implementation details are in the codebase. \Cref{fig:bottleneck_temporal} shows the estimated \gls{mi} of the contiguous process from each start state for the \cref{fig:bottleneck} \gls{mdp}s. The resulting potential contiguous empowerment field closely matches the potential empowerment field considered elsewhere in the paper in \cref{fig:bottleneck}. In both experiments, the bottleneck states have high empowerment.
\begin{figure}[t]
    \centering
    \includegraphics[width=0.6\linewidth]{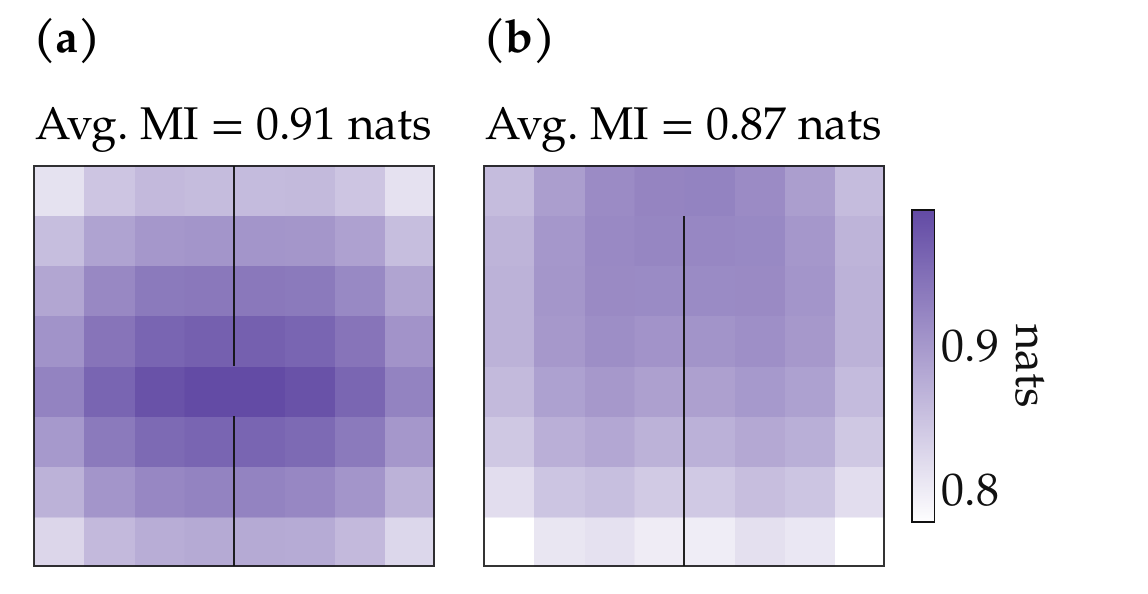}
    \caption{\textbf{Empowerment over contiguous rollouts concentrates at bottlenecks} (\cref{prop:rollout_temporal_distance})\textbf{.} The figures show the \gls{mi} of the contiguous process at each start state on the $8\times8$ grids from \cref{fig:bottleneck}, with (a) a center opening and (b) a corner opening. The bottleneck openings have high potential contiguous empowerment.}
    \label{fig:bottleneck_temporal}
\end{figure}

\section{Skill Selection: Fixed Continuous Uniform Priors are General}
\label{app:proof_skill_selection}

\begin{lemma}[A continuous uniform latent can represent any source distribution, discrete or continuous]
    \label{lem:uniform_continuous}
    Fix a start state $s_0$ and let the skill alphabet $\cZ$ be standard Borel. Consider the mutual information $I^\pi(Z;S_+\mid S_0=s_0)$ under a skill-conditioned policy $\pi \in \Pi_\cZ$, where $Z\sim p(z\mid s_0)$ is sampled once at $s_0$ and held fixed for the rollout. Let the reference source $U$ be a continuous distribution on $\cZ$, placing zero mass on every single point, e.g.\ the uniform distribution on $\cZ$. Then
    \[
        \sup_{\pi \in \Pi_\cZ,\,p(z\mid s_0)\in\Delta(\cZ)} I^\pi(Z;S_+\mid S_0=s_0)
        =
        \sup_{\tilde\pi \in \Pi_\cZ,\,\tilde Z\sim U} I^{\tilde\pi}(\tilde Z;S_+\mid S_0=s_0),
    \]
    where the maximization on the right ranges over all skill-conditioned policies.
\end{lemma}

\begin{proof}
    Fix any pair $(\pi,p)$ with $p\in\Delta(\cZ)$. By the transfer theorem (\cref{thm:transfer}; \citealp{kallenberg2021foundations}), there exists a measurable map $T_{s_0}:\cZ\to\cZ$ such that $T_{s_0}(\tilde Z)\sim p(\dmy{z}\mid s_0)$ for $\tilde Z\sim U$.\footnote{For $U=\mathrm{Unif}[0,1]$ this can be the inverse-CDF construction.} Set $Z=T_{s_0}(\tilde Z)$. Define
    \[
        \tilde\pi(a\mid s,\tilde z;s_0)=\pi(a\mid s,T_{s_0}(\tilde z);s_0).
    \]
    Conditioned on $Z=T_{s_0}(\tilde Z)$, the action kernels of $\pi$ and $\tilde\pi$ agree at every timestep and the environment dynamics are identical. Therefore, the joint law of $(Z,S_+)$ under $(\pi,p)$ is the same as the joint law of $(T_{s_0}(\tilde Z),S_+)$ under $(\tilde\pi,U)$. Moreover, $\tilde\pi$ only receives information from $\tilde z$ through $T_{s_0}(\tilde z)$, so $S_+$ is conditionally independent of $\tilde Z$ given $T_{s_0}(\tilde Z)$, and
    \[
        I^{\tilde\pi}(\tilde Z;S_+\mid S_0=s_0)
        =I^{\tilde\pi}\bigl(T_{s_0}(\tilde Z);S_+\mid S_0=s_0\bigr)
        =I^{\pi}(Z;S_+\mid S_0=s_0).
    \]
    Taking the supremum over $(\pi,p)$ shows that the RHS upper bounds the LHS. Conversely, any policy $\tilde\pi$ with source $U\in\Delta(\cZ)$ is within the domain of maximization on the LHS, so we have an equality and are done.
\end{proof}

\section{Continuous versus Discrete Skill Spaces}
\label{app:proof_discrete}

\begin{lemma}[Uniform skill priors are not generally optimal with finitely many skills]
    \label{lem:discrete_counterexample}
    With a finite skill alphabet, fixing the skill prior to be uniform can strictly reduce channel capacity.
\end{lemma}

\begin{proof}
    Consider an MDP consisting of a start state $s_0$ and two terminal states $\{s_1,s_2\}$. Let the skill set be the discrete skill set $\cZ=\{0,1\}$. Let the action space be a continuous action $a\in[0,1]$. The transition is
    \[
        p(s_1\mid s_0,a)=\beta a,
        \qquad
        p(s_2\mid s_0,a)=1-\beta a,
        \qquad 0<\beta<1.
    \]
    We can fully represent any deterministic, skill-conditioned policy by assigning each skill $z$ to an action $a_z\in[0,1]$ that completely determines the transition probability $p(s_1\mid s_0,z)=\beta a_z$.

    First optimize the channel for a fixed source. By convexity of MI in the channel (\cref{thm:mi_convexity}; \citealp{cover2006elements,polyanskiy2022information}), the maximum over $(a_0,a_1)\in[0,1]^2$ occurs at a vertex. The vertices $(0,0)$ and $(1,1)$ transmit no information. The informative vertices are $(0,1)$ and $(1,0)$, which are equivalent up to relabeling. Without loss of generality, set $(a_0,a_1)=(0,1)$, so that $p(s_1\mid s_0,z=0)=0$ and $p(s_1\mid s_0,z=1)=\beta$.

    Let $h(x)=-x\log x-(1-x)\log(1-x)$ denote the binary entropy. Under this channel, the output marginal is $p(s_1\mid s_0)=p(Z=1\mid s_0)\,\beta$, and the conditional output entropies are $h(0)=0$ and $h(\beta)$. The \gls{mi} is therefore
    \[
        I(Z;S_+\mid S_0=s_0)=h\bigl(p(Z=1\mid s_0)\,\beta\bigr)-p(Z=1\mid s_0)\,h(\beta).
    \]
    Differentiating in the source probability $p(Z=1\mid s_0)$ gives
    \[
        \frac{\partial\,I(Z;S_+\mid S_0=s_0)}{\partial\,p(Z=1\mid s_0)}=\beta\log\frac{1-p(Z=1\mid s_0)\,\beta}{p(Z=1\mid s_0)\,\beta}-h(\beta),
    \]
    so the capacity-achieving source is
    \[
        p^\star(Z=1\mid s_0)=\frac{1}{\beta(1+\exp(h(\beta)/\beta))}.
    \]
    For $\beta=1/2$, this gives $p^\star(Z=1\mid s_0)=2/5$, not $1/2$.

    To see that the uniform prior is strictly worse, maximize over channels with $p(Z=1\mid s_0)=1/2$. The same convexity argument places the optimum at vertex $(a_0,a_1)=(0,1)$ or $(1,0)$. For $\beta=1/2$, the best uniform-prior value is
    \[
        h(1/4)-\frac12 h(1/2),
    \]
    whereas the best arbitrary source value is
    \[
        h(1/5)-\frac25 h(1/2).
    \]
    The latter is strictly larger. Thus a uniform prior over two discrete skills cannot generally achieve the channel capacity.
\end{proof}

\paragraph{Comment on VIC~\citep{gregor2017variational} and DIAYN~\citep{eysenbach2018diversity}.}
A fixed continuous uniform prior is fully general for discrete and continuous \gls{mdp}s, while a discrete prior is not. This result has implications for skill learning algorithms. VIC optimizes both a source distribution and channel~\citep{gregor2017variational}, with options (termed here as ``skills'') that may be finite or continuous. However, the authors note that the joint optimization is difficult and practically implement a uniform prior over a discrete skill set. DIAYN similarly fixes the source to a uniform categorical over a discrete skill set~\citep{eysenbach2018diversity}. While neither work claims that fixing the prior preserves the objective, our result \cref{lem:discrete_counterexample} is a concrete counterexample that shows uniform discrete priors lose generality. Finally, \cref{lem:uniform_continuous} shows that utilizing a \emph{continuous} uniform skill prior recovers full generality \emph{without the need for source learning}. Combined, our results theoretically suggest that skill learning methods should utilize continuous uniform sources; note that SOTA skill learning methods already utilize such priors~\citep{zheng2025misl,park2023metra}.

\section{Empowerment Lower Bounds Tabular Adaptation}
\label{app:proof_adaptation}
We now prove that empowerment lower bounds zero-shot adaptation in the tabular setting under an uninformed reward prior.

\restatetheorem{thm:width_linear_sharp}
Define the skill offset
$u_z \triangleq \bigl(p^{\pie}(\dmy{s_+}\mid s_0,z)-p^{\pie}(\dmy{s_+}\mid s_0)\bigr)/\sqrt{b}$.
The whitened reward $g\triangleq\sqrt{b}\,r$ is standard normal under the prior, and the marginal's return $\langle g,\ p^{\pie}(\dmy{s_+}\mid s_0)/\sqrt{b}\rangle$ is mean-zero and $z$-independent, so
\begin{equation}
    J(s_0)=\E_{g}\sup_{z\in\cZ}\,\bigl\langle g,\ u_z\bigr\rangle,
    \label{eq:J_width}
\end{equation}
is the Gaussian width of $\{u_z\}_{z\in\cZ}$.

Before presenting the proof, we discuss the chosen reward prior and motivate the variance scaling. Under the $1/b$ scaling, splitting a state into substates leaves the distribution of the returns unchanged.

\begin{example}[Splitting a two-state MDP into a four-state MDP \cref{fig:state_refinement}]
    \label{ex:state_refinement}
    Let $\cS=\{1,2\}$ with uniform reference $b(1)=b(2)=\tfrac12$. The prior samples $r(1),r(2)\sim\cN(0,2)$ independently (\cref{fig:state_refinement}). For a skill $z$ with \gls{dsom} $p^{\pie}(\dmy{s_+}\mid s_0,z)=(p_1,p_2)$, the mean reward is a centered Gaussian:
    \[
        \E_{p^{\pie}(\dmy{s_+}\mid s_0,z)}[r(s_+)]
        = p_1\,r(1)+p_2\,r(2)
        \ \sim\ \cN\!\left(0,\ 2p_1^2+2p_2^2\right).
    \]
    Splitting each state into two substates gives $\cS'=\{1a,1b,2a,2b\}$ where the refined prior samples each coordinate reward $r'\sim\cN(0,4)$. Furthermore, split every skill's occupancy in the same proportions, so the occupancy of $z$ becomes $\bigl(\tfrac{p_1}{2},\tfrac{p_1}{2},\tfrac{p_2}{2},\tfrac{p_2}{2}\bigr)$. The mean reward on the fine-grained space is
    \[
        \tfrac{p_1}{2}\bigl(r'(1a)+r'(1b)\bigr)+\tfrac{p_2}{2}\bigl(r'(2a)+r'(2b)\bigr)
        \ \sim\ \cN\!\left(0,\ \tfrac{p_1^2}{4}\cdot 8+\tfrac{p_2^2}{4}\cdot 8\right)
        =\cN\!\left(0,\ 2p_1^2+2p_2^2\right),
    \]
    since each sum $r'(\cdot a)+r'(\cdot b)$ has variance $8$. Thus, the distribution of the mean reward is the same before and after splitting the state.

    In this example, the $1/b$ variance scaling is key. Under a unit-variance prior $r(s_+)\sim\cN(0,1)$, the same state splitting shrinks the return variance from $p_1^2+p_2^2$ to $\tfrac12\bigl(p_1^2+p_2^2\bigr)$. Thus, the constraint that the return distribution is scale invariant implies the above scaling of the reward prior variance.
\end{example}

\begin{figure}[t]
    \centering
    \begin{tikzpicture}[x=1cm,y=1cm]
        \node[font=\bfseries, anchor=west] at (-0.35,2.05) {(a)};
        \node at (1.15,2.05) {$\cS=\{1,2\}$};
        \draw[->] (-0.1,0) -- (2.4,0);
        \fill[black!25] (0.25,0) rectangle (0.95,1.4);
        \draw (0.25,0) rectangle (0.95,1.4);
        \fill[black!25] (1.35,0) rectangle (2.05,0.6);
        \draw (1.35,0) rectangle (2.05,0.6);
        \node at (0.6,1.65) {$p_1$};
        \node at (1.7,0.85) {$p_2$};
        \node at (0.6,-0.3) {$1$};
        \node at (1.7,-0.3) {$2$};
        \node at (1.15,-0.85) {$b\equiv\tfrac12,\ \ r(s_+)\sim\cN(0,2)$};
        \node at (1.15,-1.5) {$\E_{p}[r]\ \sim\ \cN\!\bigl(0,\,2p_1^2+2p_2^2\bigr)$};
        \draw[->,thick] (2.85,0.9) -- (4.95,0.9);
        \node[align=center] at (3.9,1.4) {refine: $b'=\tfrac{b}{2}$};
        \node[align=center] at (3.9,0.35) {occupancy splits\\[-2pt] in proportion};
        \node at (3.9,-1.5) {$\overset{d}{=}$};
        \node[font=\bfseries, anchor=west] at (4.85,2.05) {(b)};
        \node at (7.1,2.05) {$\cS'=\{1a,1b,2a,2b\}$};
        \draw[->] (5.35,0) -- (8.9,0);
        \fill[black!25] (5.5,0) rectangle (6.1,0.7);
        \draw (5.5,0) rectangle (6.1,0.7);
        \fill[black!25] (6.3,0) rectangle (6.9,0.7);
        \draw (6.3,0) rectangle (6.9,0.7);
        \fill[black!25] (7.3,0) rectangle (7.9,0.3);
        \draw (7.3,0) rectangle (7.9,0.3);
        \fill[black!25] (8.1,0) rectangle (8.7,0.3);
        \draw (8.1,0) rectangle (8.7,0.3);
        \node at (5.8,0.98) {$\tfrac{p_1}{2}$};
        \node at (6.6,0.98) {$\tfrac{p_1}{2}$};
        \node at (7.6,0.58) {$\tfrac{p_2}{2}$};
        \node at (8.4,0.58) {$\tfrac{p_2}{2}$};
        \node at (5.8,-0.3) {$1a$};
        \node at (6.6,-0.3) {$1b$};
        \node at (7.6,-0.3) {$2a$};
        \node at (8.4,-0.3) {$2b$};
        \node at (7.1,-0.85) {$b'\equiv\tfrac14,\ \ r'(s_+)\sim\cN(0,4)$};
        \node at (7.1,-1.5) {$\E_{p'}[r']\ \sim\ \cN\!\bigl(0,\,2p_1^2+2p_2^2\bigr)$};
    \end{tikzpicture}
    \caption{\textbf{State refinement leaves return distribution unchanged (\cref{ex:state_refinement}).} Bars display the occupancy $p^{\pie}(\dmy{s_+}\mid s_0,z)=(p_1,p_2)$ of a skill over the coarse states (a) and the occupancy over the split states (b). Each panel lists the reference mass and per-state reward prior $r(s_+)\sim\cN(0,1/b(s_+))$. Splitting states halves the reference mass and doubles the reward variance. Thus, the return obeys the same distribution $\cN(0,\,2p_1^2+2p_2^2)$ in both \gls{mdp}s.}
    \label{fig:state_refinement}
\end{figure}
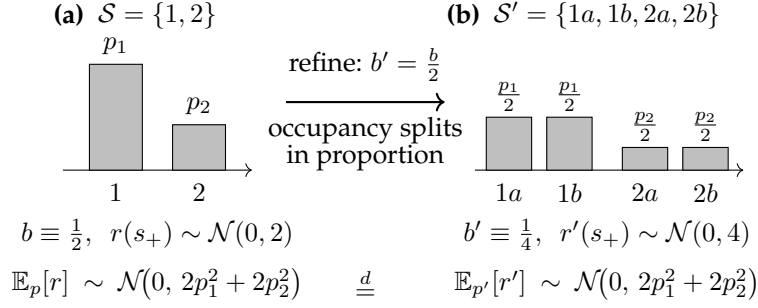

Now that we have built intuition for the choice of reward prior, we present the proof:
\begin{proof}
    By \cref{eq:J_width}, it suffices to lower bound the Gaussian width. The proof consists of two inequalities:
    \begin{equation}
        J(s_0)
        \ \overset{\text{step 1}}{\ge}\
        \frac{1}{\sqrt{2\pi}}
        \left(
        \E_{p(\dmy{z}\mid s_0)}
        \left\|u_z\right\|^2
        \right)^{1/2}
        \ \overset{\text{step 2}}{\ge}\
        \frac{I(Z;S_+\mid S_0=s_0)}{c\sqrt{2\pi}}.
        \label{eq:linear_rate_pivot}
    \end{equation}

    Fix the reward $r$ and a skill $z\in\cZ$.
    For the first inequality (step 1), we observe that the best average return obtainable by a skill dominates the marginal skill return:
    \[
        \sup_{z'\in\cZ}\E_{p^{\pie}(\dmy{s_+}\mid s_0, z')}[r(s_+)]
        \ \ge\
        \max\left(
        \E_{p^{\pie}(\dmy{s_+}\mid s_0, z)}[r(s_+)],\
        \E_{p^{\pie}(\dmy{s_+}\mid s_0)}[r(s_+)]
        \right).
    \]
    Take the expectation over rewards on both sides using the identity $\max(x,y)=\tfrac12(x+y)+\tfrac12|x-y|$. Both returns are linear in $r$, so the $\tfrac12(x+y)$ term vanishes and
    \[
        J(s_0)
        \ \ge\
        \tfrac12\,
        \E_r\left|
        \E_{p^{\pie}(\dmy{s_+}\mid s_0, z)}[r(s_+)]-\E_{p^{\pie}(\dmy{s_+}\mid s_0)}[r(s_+)]
        \right|.
    \]
    The gap in returns is an inner product of the whitened quantities:
    \[
        \E_{p^{\pie}(\dmy{s_+}\mid s_0, z)}[r(s_+)]-\E_{p^{\pie}(\dmy{s_+}\mid s_0)}[r(s_+)]
        =\left\langle g,\ u_z\right\rangle.
    \]
    Because the coordinates of $g$ are i.i.d.\ standard Gaussians, we know that the gap in returns is a Gaussian with variance $\|u_z\|^2$.
    Hence, for every $z\in\cZ$,
    \begin{align*}
        J(s_0)
        \  & \ge\
        \tfrac12\,
        \E_r\left|\left\langle g,\ u_z\right\rangle\right| \\
        \  & =\
        \frac{1}{\sqrt{2\pi}}
        \left\|u_z\right\|.
    \end{align*}
    Averaging the square of this bound over $z\sim p(\dmy{z}\mid s_0)$ (LHS is independent of the skill $z$) gives step 1 of \cref{eq:linear_rate_pivot}. For step 2, express the \gls{mi} as the expected KL divergence between the skill prior and the skill posterior after observing $s_+$:
    \[
        I(Z;S_+\mid S_0=s_0)
        =\sum_{s_+}p^{\pie}(s_+\mid s_0)\,
        \KL\bigl(p^{\pie}(\dmy{z}\mid s_0,s_+)\,\big\|\,p(\dmy{z}\mid s_0)\bigr).
    \]
    Applying Bayes' rule, the $\chi^2$ divergence takes the form of a mean squared deviation of the skill \gls{dsom}s from the marginal:
    \begin{align*}
        \chi^2\bigl(p^{\pie}(\dmy{z}\mid s_0,s_+)\,\big\|\,p(\dmy{z}\mid s_0)\bigr)
         & =\E_{p(\dmy{z}\mid s_0)}\left[\left(\frac{p^{\pie}(s_+\mid s_0, z)}{p^{\pie}(s_+\mid s_0)}-1\right)^{2}\right]
        \\
         & =\frac{\E_{p(\dmy{z}\mid s_0)}\bigl[\bigl(p^{\pie}(s_+\mid s_0, z)-p^{\pie}(s_+\mid s_0)\bigr)^{2}\bigr]}{p^{\pie}(s_+\mid s_0)^{2}} .
    \end{align*}
    Then, through $\chi^2$, we can upper bound the \gls{mi} with the norm of the occupancy gap vector:
    \begin{align*}
        I(Z;S_+\mid S_0=s_0)
         & \le
        c\sum_{s_+}p^{\pie}(s_+\mid s_0)
        \sqrt{\chi^2\bigl(p^{\pie}(\dmy{z}\mid s_0,s_+)\,\big\|\,p(\dmy{z}\mid s_0)\bigr)} \tag{\cref{thm:kl_chi2}}
        \\
         & =
        c\sum_{s_+}\sqrt{\E_{p(\dmy{z}\mid s_0)}\bigl[\bigl(p^{\pie}(s_+\mid s_0, z)-p^{\pie}(s_+\mid s_0)\bigr)^{2}\bigr]}
        \\
         & =
        c\sum_{s_+}\sqrt{b(s_+)}\,
        \sqrt{\frac{\E_{p(\dmy{z}\mid s_0)}\bigl[\bigl(p^{\pie}(s_+\mid s_0, z)-p^{\pie}(s_+\mid s_0)\bigr)^{2}\bigr]}{b(s_+)}}
        \\
         & \le
        c\left(\sum_{s_+}\frac{\E_{p(\dmy{z}\mid s_0)}\bigl[\bigl(p^{\pie}(s_+\mid s_0, z)-p^{\pie}(s_+\mid s_0)\bigr)^{2}\bigr]}{b(s_+)}\right)^{1/2} \tag{Cauchy-Schwarz}
        \\
         & =
        c\left(\E_{p(\dmy{z}\mid s_0)}\left\|u_z\right\|^{2}\right)^{1/2}.
    \end{align*}
    The first line applies \cref{thm:kl_chi2} with the bound $\ln(1+x)\le c\sqrt{x}$ to each divergence. The second line expands $\chi^2$ and cancels the marginal weight. The third and fourth lines apply Cauchy--Schwarz. The last equality is the definition of $u_z$, proving step 2 of \cref{eq:linear_rate_pivot} which finishes the proof.
\end{proof}

We provide some didactic tabular experiments that demonstrate the behavior and scaling of the bound over different \gls{mdp}s.
In \cref{fig:kl_width_J_quality}, we demonstrate the relationship between the degree of commitment within an \gls{mdp} -- the controllability of the \gls{mdp} -- and the bound for a fixed set of states.
As controllability of the \gls{mdp} increases, the bound tightens.
In \cref{fig:kl_width_J_vs_bound}, we demonstrate the relationship between the number of absorbing target states in an \gls{mdp} and the bound under fixed commitment.
As the number of absorbing target states grows, the bound loosens.

\begin{figure}[t]
    \centering
    \includegraphics[width=0.6\linewidth]{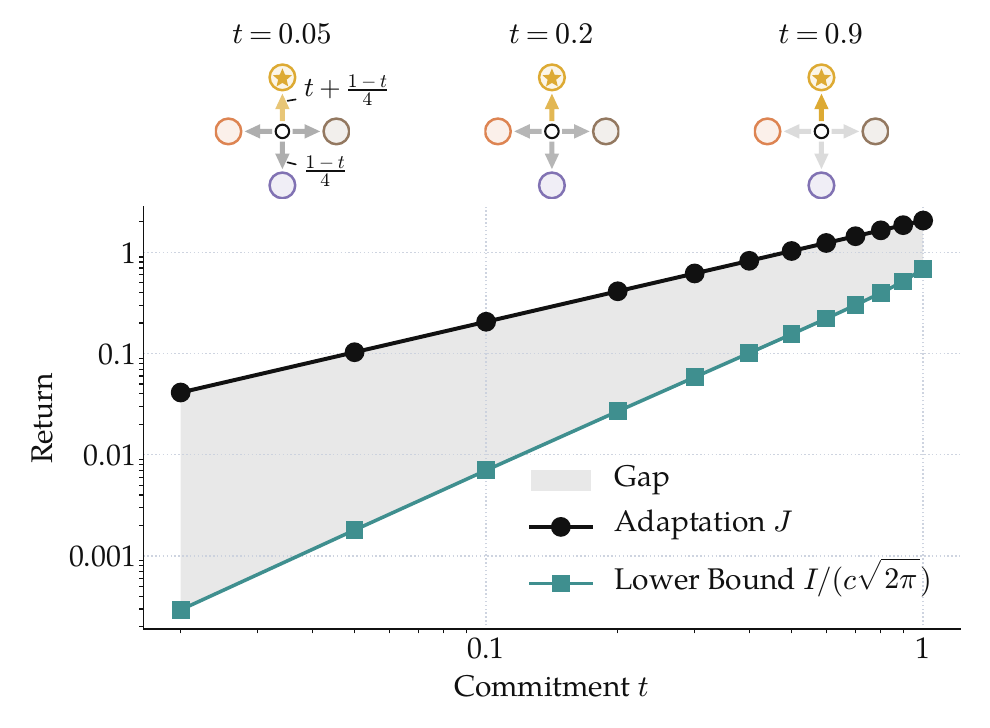}
    \caption{\textbf{\gls{mi} bound as a function of skill commitment} (\cref{thm:width_linear_sharp})\textbf{.} Consider the \gls{mdp} consisting of 5 states with a central starting state, where any left/right/up/down action sends probability mass to absorbing left/right/up/down states. Actions left/right/up/down have a level of commitment $t$, which measures the additional mass above uniform assigned to the corresponding state. When commitment is small, probability is almost spread equally over all four states regardless of action so there is little control. In this regime, the adaptation $J$ dominates the lower bound. As the commitment increases, actions/skills commit to states and the bound tightens.}
    \label{fig:kl_width_J_quality}
\end{figure}

\begin{figure}[t]
    \centering
    \includegraphics[width=0.6\linewidth]{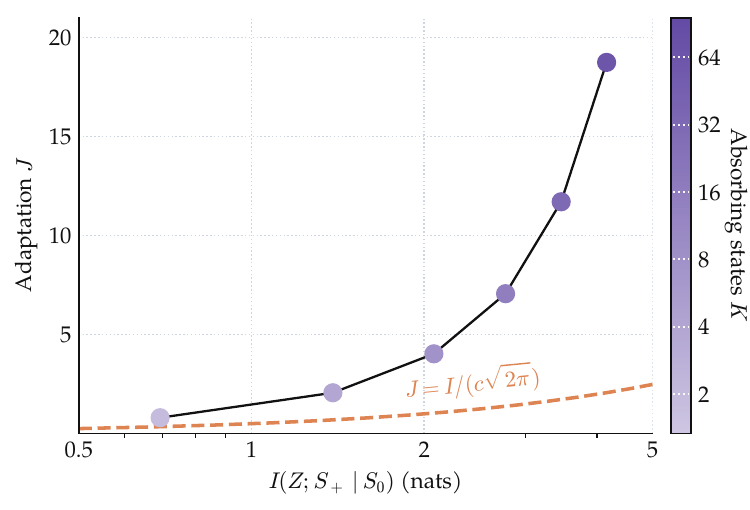}
    \caption{\textbf{\gls{mi} bound as a function of absorbing states} (\cref{thm:width_linear_sharp})\textbf{.} Consider an \gls{mdp} with one starting state and $K$ absorbing target states. Each of the $K$ actions transitions deterministically to a distinct target. \textbf{We normalize occupancies over the targets (neglecting the occupancy of the starting state).} Thus, $K$ counts the maximum possible number of total skills. The scaled \gls{mi} (orange) correctly lower bounds the adaptation $J$ (purple). The bound weakens as the number of absorbing target states in the \gls{mdp} grows.}
    \label{fig:kl_width_J_vs_bound}
\end{figure}

\section{When Information Geometry is \emph{Insufficient} for Adaptation}
\label{app:proof_continuous_collapse}

The adaptation result above is in the tabular setting.
For the following results, we also consider \gls{mdp}s whose state space $\cS$ is a metric space with arbitrary metric $d$. Importantly, these state spaces may be continuous.

\begin{proposition}[\gls{mdp}s can admit arbitrarily many distinct skills]
    \label{prop:admits_all_k}
    Fix a discount $\gamma\in(0,1)$ and a skill space $\cZ=[K]$ with the uniform prior $Z\sim\mathrm{Unif}([K])$. There exists an \gls{mdp} with a start state $s\in\cS$ and $K$ skills that attains the potential empowerment
    \[
        \gEp(s)=\max_{\pie} I^{\pie}(Z;S_+\mid S_0=s)=\gamma\log K .
    \]
    If $\cS$ is continuous, there exists an \gls{mdp} where a continuum of these skills gives unbounded potential empowerment $\gEp(s)=+\infty$.
\end{proposition}

\begin{proof}

    \emph{Discrete case.}
    We construct \gls{mdp}s that admit $\gEp(s) \geq \gamma \log K$ for all $K$.

    Consider an \gls{mdp} that admits a large number of distinguishable occupancies. Then, we can assign each skill a distinct occupancy at some target future state. In this limit, the skill becomes readily identifiable from the future state.

    For finite $K$, consider a discretized state space $\cS$ with states $s,s_1,\ldots,s_K \in \cS$ where all states $s_1, \ldots, s_K$ are absorbing. Let skill $z$ immediately transition from start state $s$ to another target state. Then, the induced \gls{dsom} is
    \[
        p^{\pie}(\dmy{s_+}\mid s,z)=(1-\gamma)\,\delta_{s}+\gamma\,\delta_{s_z},
    \]
    where $\delta$ is the Kronecker delta. Here, the skills share the start state $s$ with probability mass $1-\gamma$ and are otherwise disjoint over $\cS\setminus\{s\}$. Observing $S_+ \neq s$ exactly determines the skill while observing $S_+ = s$ gives no information. Thus, the \gls{mi} is
    \[
        I^{\pie}(Z;S_+\mid S_0=s)=\gamma\,H(Z)=\gamma\log K,
    \]
    which obtains the potential empowerment.

    \emph{Continuous case.} We similarly construct  \gls{mdp}s that admit unbounded $\gEp(s)$. When $\cS$ is continuous, we can index skills by a continuous parameter rather than a finite set and map $z \in [0, 1]$ to $S \setminus \{s\}$. The induced \gls{dsom} is
    \[
        p(\dmy{s_+}\mid s,z)=(1-\gamma)\,\delta_s+\gamma\,\delta_{\phi(z)}
    \]
    where $\delta$ is now the Dirac delta. As before, any final state that is not the start state identifies the skill. If we cut $[0,1]$ into $n$ equal segments, the future state $S_+$ reveals at least $\gamma\log n$ bits about the skill. Since this construction holds for all $n$, we have that
    \[
        I(Z;S_+\mid S_0=s)\ge\sup_n \gamma\log n=+\infty .
    \]
    Thus, we have constructed an \gls{mdp} where the \gls{mi} is unbounded for every $\gamma\in(0,1)$ for continuous state spaces.
\end{proof}

We can utilize this proposition to prove that potentially infinite bits can fit within a tiny metric ball:
\begin{proposition}[Many bits can fit in a tiny metric ball]
    \label{prop:mi_lipschitz_collapse}
    Fix a start state $s$ and a ball $B(\bar s,\epsilon)$ of radius $\epsilon$ with infinitely many states, centered at an arbitrary state $\bar s\in\cS$. For every $K$, there exists an \gls{mdp} with $K$ skills as in \cref{prop:admits_all_k} whose target states $s_z$ all lie inside the ball, so
    \[
        I(Z;S_+\mid S_0=s)=\gamma\log K,
    \]
    yet every bounded 1-Lipschitz reward $r$ gives all skills nearly the same return:
    \[
        \sup_{z,z'}\left|
        \E_{p(\dmy{s_+}\mid s,z)}[r(S_+)]-
        \E_{p(\dmy{s_+}\mid s,z')}[r(S_+)]
        \right|
        \le 2\gamma\epsilon
        \le 2\epsilon.
    \]
    With a continuum of skills inside the ball, the \gls{mi} can be unbounded while the upper bound on the return difference between skills is unchanged.
\end{proposition}

\begin{proof}
    We can apply \cref{prop:admits_all_k} to the ball $B(\bar s, \epsilon)$ and construct valid \gls{mdp}s with the above \gls{mi}s. For any 1-Lipschitz reward, the skill returns of $z \neq z'$ in these \gls{mdp}s differ by
    \begin{align*}
        \left|
        \E_{p(\dmy{s_+}\mid s,z)}[r(S_+)]-
        \E_{p(\dmy{s_+}\mid s,z')}[r(S_+)]
        \right|
         & =\gamma\,|r(s_z)-r(s_{z'})|                            \\
         & \le \gamma\,\bigl(d(s_z,\bar s)+d(\bar s,s_{z'})\bigr) \\
         & \le 2\gamma\epsilon
    \end{align*}
    which proves the result.
\end{proof}

For any 1-Lipschitz reward, it is always possible for a set of skills to have high \gls{mi} (potentially unbounded) while providing no adaptation advantage for continuous rewards beyond a single policy. Thus,
we should \emph{not} extrapolate the tabular adaptation bound (\cref{thm:width_linear_sharp}) to continuous spaces with smooth rewards.

We show an example of this collapse in \cref{fig:mi_ball_collapse}. The reachable future state occupancies pack closer together in continuous space from left to right. In the limit, the difference in values between skills shrinks to zero while the \gls{mi} diverges, so the discrete adaptation bound becomes trivial. We note that prior works explicitly point out and propose solutions to the problem of diverging channel capacities in continuous spaces~\citep{tiomkin2017control,salge2013continuous}, though the works do not formally connect empowerment and adaptation.

\section{Anisotropic Reward Priors}
\label{app:anisotropic}

This section relaxes the reward isotropy assumption in the adaptation result \cref{thm:width_linear_sharp}. Unfortunately, using an anisotropic reward prior with covariance $\Sigma$ scales the lower bound by $\sqrt{\lambda_{\min}(\Sigma)}$ (\cref{prop:anisotropic_degraded}). In the limit, skills with arbitrarily large mutual information can exhibit no adaptation even with fixed reward power constraints (e.g. fixed $\Tr{\Sigma}$; see \cref{prop:anisotropic_collapse}).

We keep the discrete setting of \cref{thm:width_linear_sharp} and only change the reward prior. We sample rewards with $\sqrt{b}\,r\sim\cN(0,\Sigma)$, so the law of the uniform baseline-whitened reward $g$ in \cref{eq:J_width} becomes $\cN(0,\Sigma)$. An arbitrary anisotropic prior may concentrate variance on directions that the empowerment-maximizing skills choose to ignore. The marginal's return remains mean zero for any $\Sigma$, so the adaptation objective takes the form
\[
    J^\Sigma(s_0)
    \triangleq
    \E_{g\sim\cN(0,\Sigma)}
    \left[\sup_{z\in\cZ}\bigl\langle g,\ u_z\bigr\rangle\right].
\]

\makerestatable
\begin{proposition}[Anisotropic reward prior scales bound]
    \label{prop:anisotropic_degraded}
    In the setting of \cref{thm:width_linear_sharp}, sample rewards with $\sqrt{b}\,r\sim\cN(0,\Sigma)$ instead. Then, the \gls{mi} bound rescales by $\sqrt{\lambda_{\min}(\Sigma)}$:
    \[
        J^\Sigma(s_0)
        \;\ge\;
        \sqrt{\lambda_{\min}(\Sigma)}\;
        \frac{I(Z;S_+\mid S_0=s_0)}{c\sqrt{2\pi}}.
    \]
    The isotropic prior $\Sigma=\mathrm{Id}$ recovers \cref{thm:width_linear_sharp} exactly.
\end{proposition}

\begin{proof}
    The proof of \cref{thm:width_linear_sharp} uses the reward prior to express a lower bound on $J (s_0)$ in terms of the norm of the occupancy gap vector. Under the $\cN(0,\Sigma)$ prior, the occupancy gap vector variance scales by the anisotropic covariance. Then, the result from \cref{thm:width_linear_sharp} suggests the following modified bound:
    \begin{equation}
        J^\Sigma(s_0)
        \ \overset{\text{step 1}}{\ge}\
        \frac{1}{\sqrt{2\pi}}
        \left(
        \E_{p(\dmy{z}\mid s_0)}
        \left\|\Sigma^{1/2}u_z\right\|^2
        \right)^{1/2}
        \ \overset{\text{step 2}}{\ge}\
        \sqrt{\lambda_{\min}(\Sigma)}\;
        \frac{I(Z;S_+\mid S_0=s_0)}{c\sqrt{2\pi}}.
        \label{eq:linear_rate_pivot_anisotropic}
    \end{equation}
    Now, we carefully derive this bound. The reduction of $J$ to the gap in returns in \cref{thm:width_linear_sharp} only depends on a centered reward. Thus, as before, we have that
    \[
        J^\Sigma(s_0)
        \ \ge\
        \tfrac12\,
        \E_{g}\left|\left\langle g,\ u_z\right\rangle\right|.
    \]
    For $g\sim\cN(0,\Sigma)$, the return gap $\langle g,\ u_z\rangle$ is Gaussian with variance $\bigl\|\Sigma^{1/2}u_z\bigr\|^2$.
    Hence, for every $z\in\cZ$,
    \[
        J^\Sigma(s_0)
        \ \ge\
        \frac{1}{\sqrt{2\pi}}
        \left\|\Sigma^{1/2}u_z\right\|.
    \]
    Averaging the square of this bound over $z\sim p(\dmy{z}\mid s_0)$ (the left side does not depend on $z$) gives step 1. We can lower bound the norm:
    \[
        \|\Sigma^{1/2}u_z\|^2
        =\langle u_z,\Sigma u_z\rangle
        \;\ge\;
        \lambda_{\min}(\Sigma)\,\|u_z\|^2 .
    \]
    Combining with the step 2 bound (does not change with the reward prior) gives the desired result.
\end{proof}

Indeed, we can construct an example where a state and set of policies with high potential empowerment exhibit zero adapation to anisotropic rewards:
\begin{proposition}[Large \gls{mi} with zero anisotropic adaptation]
    \label{prop:anisotropic_collapse}
    Fix a power budget $T>0$. For every $K$, there exists a discrete \gls{mdp} with skills achieving $I(Z;S_+\mid S_0=s_0)=\gamma\log K$ and a reward prior $\sqrt{b}\,r\sim\cN(0,\Sigma)$ with $\Tr{\Sigma}=T$ such that $J^\Sigma(s_0)=0$.
\end{proposition}

\begin{proof}
    The discrete construction in \cref{prop:admits_all_k} attains $I(Z;S_+\mid S_0=s_0)=\gamma\log K$ with $K$ occupancy vector offsets $u_z$ that span a subspace of dimension at most $K-1$ of the $(K{+}1)$-dimensional whitened reward space.

    Choose a unit vector $v$ orthogonal to every $u_z$ and set the reward variance $\Sigma=T\,vv^\top$. This choice places all variance of the reward prior in the space orthogonal to the span of the occupancy gap vectors. By the proof of \cref{prop:anisotropic_degraded}, each return gap $\langle g,u_z\rangle$ is Gaussian with variance $\|\Sigma^{1/2}u_z\|^2=0$, so $J^\Sigma(s_0)=0$ while the \gls{mi} is unchanged.
\end{proof}

Thus, while the \gls{mi} alone can provide nontrivial lower bounds on adaptation in tabular settings for an isotropic reward prior, the same no longer holds true for anisotropic reward priors. An empowerment-maximizing agent could learn skills with high control over observations and ignore other learnable skills that control the rewarding degrees of freedom.

\section{Tabular Experiment Details}
\label{app:tab_exps}

\subsection{Tabular $5\times5$ grid experiments}

\label{app:5x5}
\Cref{tab:5x5_hyperparams} details the $5\times 5$ grid \gls{mdp} parameters used for \cref{fig:eff_emp_5x5}.

\paragraph{Skill-learning via convex optimization.} In the state centrality experiments, we used a convex solver to enumerate vertices of the reachable discounted state occupancy measure (\gls{dsom}) polytope for all possible starting states. An \gls{mdp} with $|\mathcal{S}|$ states and $|\mathcal{A}|$ actions can have up to $|\mathcal{A}|^{|\mathcal{S}|}$ vertices. Thus, enumerating all vertices is challenging~\citep{eysenbach2021information}.

However, we can use the fact that reward-maximizing \gls{dsom} always sit at the vertices of the convex polytope to conduct vertex search. Specifically, we (1) sample a random reward vector $\theta \in \Delta(\mathcal{S})$ in the state simplex and (2) use a convex solver to maximize \gls{dsom} alignment with the reward vector subject to probability flow constraints~\citep{puterman2014markov}. Since a linear objective over the reachable state polytope is maximized at an extreme point, and a generic random objective exposes a vertex almost surely, this procedure retrieves valid vertices of the polytope conditioned on specific starting states. This procedure generally does not retrieve the complete set of candidate vertices over which to specify skills, but it provides a good approximation.

Policies that optimize
\[\max_{\pi(a \mid s, z; s_0),\, p(z \mid s_0)} I(Z;S_+ \mid S_0 = s_0),\]
seek out the vertices of the polytope (see \cref{thm:effective_info_geom}). Thus, we fix the approximate feasible set of optimal policies to the set of found vertices. This procedure constructs a one-to-one mapping between skills and vertex \gls{dsom}s. Then, getting the empowerment-maximizing policy with known dynamics is simple: it is well known that for a fixed channel encoder $p^\pi(s_+ \mid s_0, z)$, the \gls{mi} is concave in the source distribution $p(z \mid s_0)$ and convex in the channel (\cref{thm:mi_convexity}; \citealp{cover2006elements}). Thus, finding the optimal source distribution $p(z \mid s_0)$ is reduced to a convex optimization problem solvable with standard convex optimization packages (see the Blahut--Arimoto algorithm; \citealp{arimoto1972algorithm, blahut1972computation}).

The results in \cref{fig:eff_emp_5x5} show that the latent $z$-conditioned policies that start at $s_0 = (2,2)$ cover all states within the \gls{mdp}.

\paragraph{Learning empowerment-maximizing policies.} The above procedure optimizes the channel capacity over the sampled vertex set, yielding a lower bound on
\[\gEp(s_0) = \max_{\pi(a \mid s, z; s_0),\, p(z \mid s_0)}I(Z;S_+ \mid S_0 = s_0)\]
that tightens as the candidate vertex set grows. We use standard Q-learning methods to obtain the greedy policy that maximizes
\[
    Q^{\pip}(s,a) = \mathbb{E}\left[\sum_{t=0}^{\infty} \gamma^t \gEp (s_t) \middle| s_0 = s, a_0 = a \right],
\]
which we refer to as the potential empowerment $Q$-value. We anneal the epsilon greedy action rate from $1.0 \rightarrow 0.01$ from the start to the end of Q-learning. The optimal policy shown in \cref{fig:eff_emp_5x5}(c) moves to the center of the room, maximizing reachability to all other states in the \gls{mdp}.

\begin{table}[htb]
    \centering
    \caption{Hyperparameters for the $5\times5$ gridworld experiments (\cref{fig:eff_emp_5x5})}
    \begin{tabular}{@{}p{0.42\linewidth}p{0.50\linewidth}@{}}
        \toprule
        \textbf{Parameters}                              & \textbf{Value}                                \\
        \midrule
        Environment size (n $\times$ m)                  & 5 $\times$ 5                                  \\
        Environment actions                              & \{Left, Right, Up, Down, Stay\}               \\
        Periodic boundary conditions                     & No                                            \\
        $P(\text{random action})$                        & 0.0                                           \\
        Gamma for \gls{dsom}                             & 0.95                                          \\
        \midrule
        \# uniformly sampled vertex directions per $s_0$ & 500                                           \\
        Convex solver                                    & Clarabel~\citep{goulart2024clarabel}          \\
        Blahut--Arimoto max iterations                                & \texttt{50\_000}                              \\
        Blahut--Arimoto tolerance (capacity gap)                                  & $1$e$-6$                                      \\
        Vertex seed                                      & 0                                             \\
        \midrule
        Epsilon schedule                                 & Linear interpolation from $1\rightarrow 0.01$ \\
        Gamma for potential policy                       & 0.95                                          \\
        \# Q-learning iterations                         & $2\times 10^5$                                \\
        Q-learning step size $\alpha$ & 0.1 \\
        \midrule
    \end{tabular}
    \label{tab:5x5_hyperparams}
\end{table}

\subsection{Bottleneck grid experiments}
\label{app:bottleneck_exps}

\Cref{fig:bottleneck} uses the same vertex search and source-distribution optimization as \cref{app:5x5}. The only difference is that in this GridWorld, the map is now $8 \times 8$ and has impenetrable walls within the state space. We consider $\gamma\in\{0.5,0.95\}$.

This \gls{mdp} demonstrates the bottleneck effect discussed in \cref{subsec:bottlenecks}. In the open grid, the center state has maximal potential empowerment. With the added walls, the bottleneck openings at the gaps in the walls have high empowerment. At these states, the agent's choice in skill substantially changes its temporal distance to future states.

\subsection{Key-finding experiments}

\label{app:key}

\Cref{tab:key_hyperparams} gives the key-finding experiment hyperparameters.

We use a numerical experimental setup identical to that of the $5\times 5$ room with walls. In this \gls{mdp}, the agent can pick up and drop off a key at the square $(0,2)$. Picking up this key allows access to the blue-shaded regions of the \gls{mdp}; without the key, the agent cannot enter these spaces in the grid.

\cref{fig:pot_emp} shows the resulting empowerment-maximizing policy that ascends the empowerment landscape. Starting at state $s_0 = (4,2)$ without a key, the agent learns to navigate to the key, pick it up, and navigate to the center square $(2,2)$. This experiment shows that states with maximal empowerment capture an information-geometric notion of centrality---rather than simply navigating to and staying at the {\it physically} center state, as in the $5 \times 5$ grid, the agent recognizes that the most central state is the one where it has obtained the key.

\begin{table}[htb]
    \centering
    \caption{Hyperparameters for the $5\times5$ gridworld with walls and a key}
    \begin{tabular}{@{}p{0.42\linewidth}p{0.50\linewidth}@{}}
        \toprule
        \textbf{Parameters}                              & \textbf{Value}                                                               \\
        \midrule
        Environment size (n $\times$ m)                  & 5 $\times$ 5                                                                 \\
        Environment actions                              & \{Left, Right, Up, Down, Stay, Pick up key at $(0,2)$, Drop key at $(0,2)$\} \\
        Periodic boundary conditions                     & No                                                                           \\
        $P(\text{random action})$                        & 0.0                                                                          \\
        Gamma for skill-learning                         & 0.99                                                                         \\
        \midrule
        \# uniformly sampled vertex directions per $s_0$ & 1000                                                                         \\
        Convex solver                                    & Clarabel~\citep{goulart2024clarabel}                                         \\
        Blahut--Arimoto max iterations                                & \texttt{50\_000}                                                             \\
        Blahut--Arimoto tolerance (capacity gap)                                  & $10^{-6}$                                                                     \\
        Vertex seed                                      & 0                                                                            \\
        \midrule
        Epsilon schedule                                 & Linear interpolation from $1\rightarrow 0.01$                                \\
        Gamma for potential policy                       & 0.99                                                                         \\
        \# Q-learning iterations                         & $2\times 10^5$                                                               \\
        Q-learning step size $\alpha$ & 0.1 \\
        \midrule
    \end{tabular}
    \label{tab:key_hyperparams}
\end{table}

\begin{figure}
    \centering
    \includegraphics[width=0.75\linewidth]{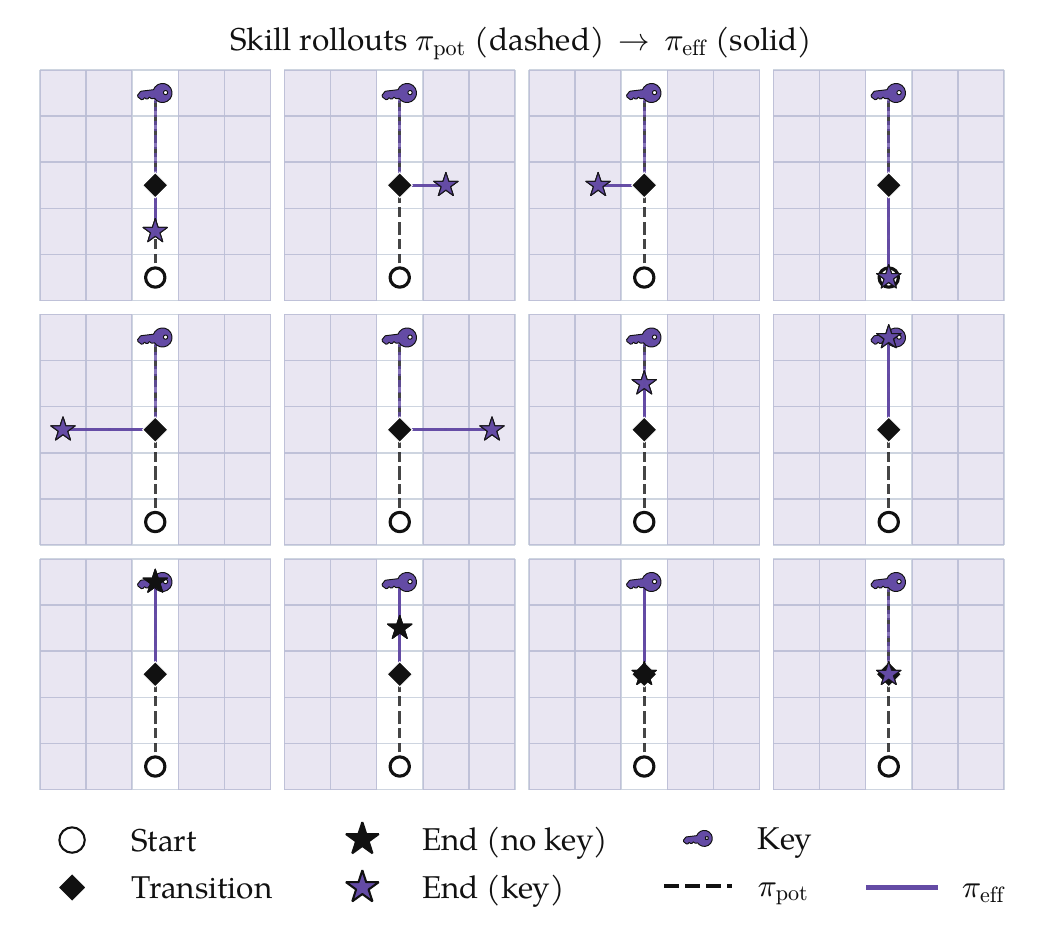}
    \caption{\textbf{High empowerment finds central states, where skills are diverse from the central states} (\cref{thm:effective_info_geom})\textbf{.} Potential policy navigates to central states. Learned skills spread out from the central state and measure the degree of centrality of the state (in information geometry). The dotted line denotes the potential empowerment policy's trajectory for $\frac{1}{1-\gamma}$ steps. Diamond denotes the handoff state $s_0$, where the agent randomly samples $z \sim p(z\mid s_0)$ and rolls out skill policy $\pi(a \mid s, z; s_0)$.}
    \label{fig:pot_emp}
\end{figure}
\FloatBarrier

\subsection{Adaptation geometry examples}
\label{app:adaptation_geometry_exps}

In the example \gls{mdp} of \cref{fig:adaptation_geometry}, every skill transitions in one step from the starting state to an absorbing target state. Thus, a skill's \gls{dsom}, normalized over the targets, is the transition distribution. \Cref{tab:adaptation_geometry_mdps} lists the resulting occupancy distributions.

In Example~X (\cref{fig:adaptation_geometry}(a)), there are $N_T=2$ targets, and each of two skills transitions to its target with probability one. In Example~Y (\cref{fig:adaptation_geometry}(d)), all three skills reach $N_T=3$ targets and no others. Each skill transitions to its target with probability $\tfrac{23}{30}$ and to each of the other two targets with probability $\tfrac{7}{60}$.

\begin{table}
    \caption{Skill occupancy distributions of the example \gls{mdp}s. (a) Example~X with mass on the two reachable targets. (b) Example~Y with mass on the three reachable targets.}
    \label{tab:adaptation_geometry_mdps}
    \centering
    \footnotesize
    \begin{subtable}[t]{0.4\linewidth}
        \centering
        \subcaption{\Cref{fig:adaptation_geometry}(a) \gls{mdp}}
        \begin{tabular}{@{}lcc@{}}
            \toprule
                  & $s_1$ & $s_2$ \\
                  \midrule
            $z=1$ & $1$ & $0$ \\
            $z=2$ & $0$ & $1$ \\
            \bottomrule
        \end{tabular}
    \end{subtable}
    \hfill
    \begin{subtable}[t]{0.4\linewidth}
        \centering
        \subcaption{\Cref{fig:adaptation_geometry}(d) \gls{mdp}}
        \begin{tabular}{@{}lccc@{}}
            \toprule
                  & $s_1$ & $s_2$ & $s_3$ \\
            \midrule
            $z=1$ & $23/30$ & $7/60$ & $7/60$ \\
            $z=2$ & $7/60$ & $23/30$ & $7/60$ \\
            $z=3$ & $7/60$ & $7/60$ & $23/30$ \\
            \bottomrule
        \end{tabular}
    \end{subtable}
    \hfill

\end{table}

\Cref{fig:adaptation_geometry}(b,~e) plot the polytope of the whitened state occupancy offsets $u_z$ in \cref{eq:J_width} with the uniform baseline $b=1/N_T$ over the targets.

The overlap between Example~Y's skills leads to low channel capacity.
However, the spread over the target states gives a larger Gaussian width.
The orderings of the adaptation and the obtained \gls{mi} therefore flip: X attains the higher mutual information and Y the larger width.

In the Gaussian width panels (\cref{fig:adaptation_geometry}(b,~e)), the offsets in \cref{eq:J_width} span at most two dimensions, and the axes $(u_1, u_2)$ give the coordinates of $u$ in an orthonormal basis of this span.

At most two offsets are linearly independent in each example in \cref{fig:adaptation_geometry}. The occupancy offsets in the top row are negatives of each other under the uniform skill prior and the bottom row's three offsets satisfy $u_1+u_2+u_3=0$. Thus, we can draw a 2D diagram of these geometries (\cref{fig:adaptation_geometry}(b,~c,~e,~f)); the insets in \cref{fig:adaptation_geometry}(f) visualize the occupancy distributions. Example~X's antipodal pair spans only one degree of freedom, so \cref{fig:adaptation_geometry}(b) shows a segment. In addition to the possible occupancies, the coordinates in this space can also denote reward vectors, and the background color indicates the skill that adapts to that reward. In the information panels, each colored endpoint is a possible skill's \gls{dsom}, at the vertices of a reachable polytope in \cref{fig:adaptation_geometry}. The skill \gls{dsom}s and their mixture are shown explicitly in the inset bar charts. Arrows point from the capacity-achieving mixture $q^\star=b$ (black dot) to every possible skill's \gls{dsom}, with equal positive source weights. The solid black contour in \cref{fig:adaptation_geometry}(f) traces the \gls{kl} level set at the capacity. In \cref{fig:adaptation_geometry}(c), this level set is just two endpoints.

\subsection{Worked example of centrality in information geometry}
\begin{wrapfigure}{R}{0.5\linewidth}
    \centering
    \includegraphics[width=\linewidth]{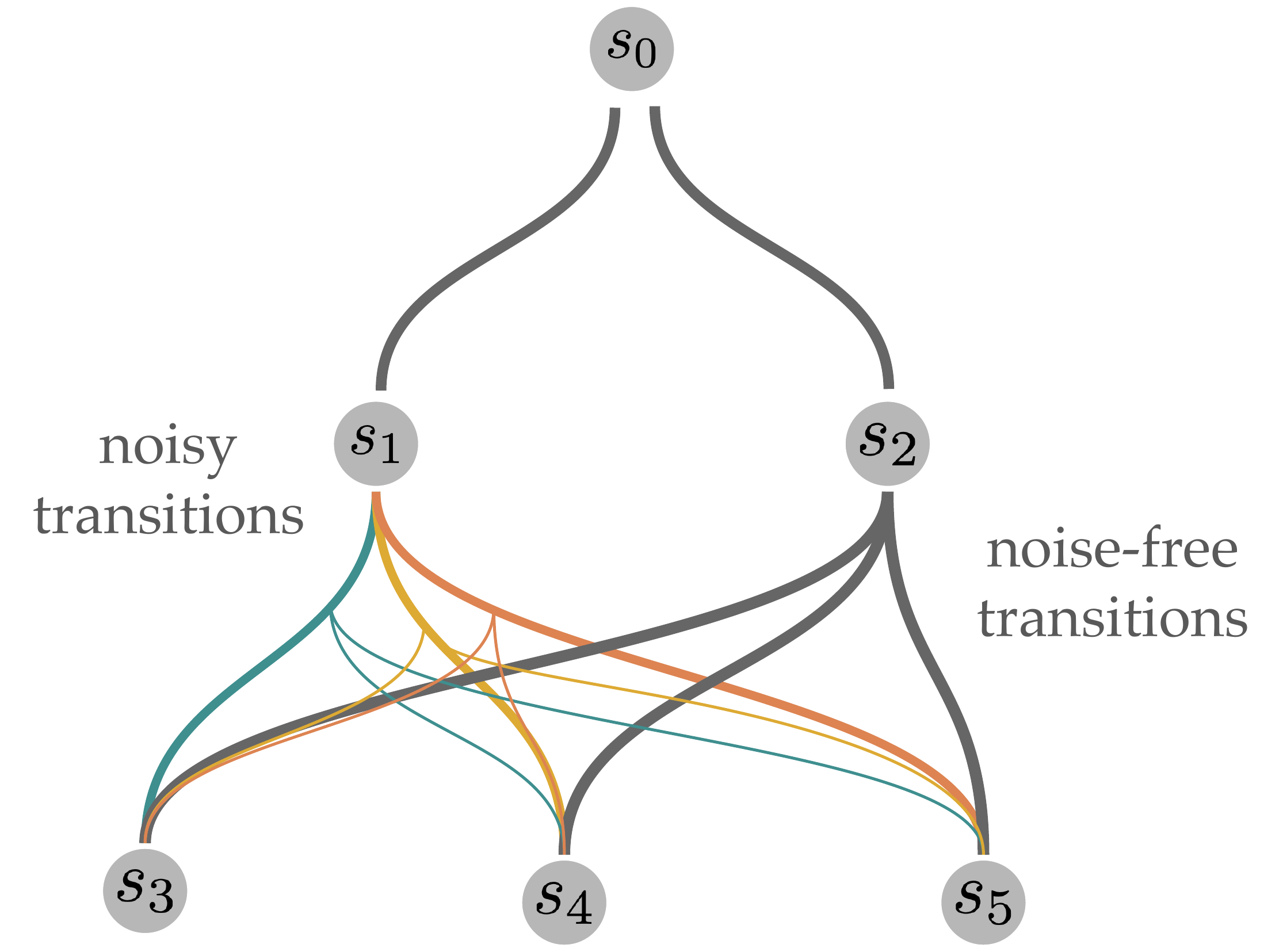}
    \caption{\textbf{Worked Example of Empowerment.}}
\end{wrapfigure}

We illustrate the geometry of empowerment in \cref{fig:example-left}. To start, consider learning skills that start at state $s_2$. From this state, there are three actions which deterministically transition to $s_3$, $s_4$, or $s_5$. We consider learning three skills, so $z \in \{1, 2, 3\}$. The problem of optimizing mutual information at state $s_2$ corresponds to learning the policy $\pi(a \mid s = s_2, z)$ so that the skill index $z$ conveys as many bits as possible about the next state. A poor choice for this policy would be the random policy, wherein each skill selects randomly among the three actions. Each of these three skills induces the same distribution over future states $s_3, s_4, s_5$, conveying zero bits of information. We visualize the induced state occupancies in \cref{fig:example-left}. Conversely, if we assign each skill a unique action (e.g., $z = 1$ always takes the leftmost action), then the state distributions the skills induce become markedly different. As visualized in \cref{fig:example-left}, this choice of skills covers a much larger region.

In state $s_1$, the transitions are stochastic. If we similarly learned a set of skills to maximize mutual information, they would arrive at occupancy measures denoted by the orange dots in \cref{fig:example-left}. We can now interpret this example through the lens of empowerment. Is the empowerment $\gEp(s)$ higher in state $s_1$ or $s_2$? Looking at \cref{fig:example-left}, we see that the skills at state $s_2$ are more ``spread out''; formally, they convey more bits to the future. Thus, an empowerment-maximizing policy would prefer to navigate towards and stay at state $s_2$.

\end{document}